%% file: acl.tex
\documentclass[11pt]{article}

\usepackage{amsthm}
\usepackage[final]{acl}

\usepackage{times}
\usepackage{latexsym}

\usepackage[T1]{fontenc}

\usepackage[utf8]{inputenc}

\usepackage{microtype}

\usepackage{inconsolata}

\usepackage{graphicx}

\usepackage{listings}
\usepackage{fancyvrb}
\usepackage{fvextra}
\usepackage{enumitem}
\usepackage{standalone}
\usepackage{newtxmath}
\usepackage{wrapfig}
\usepackage{subcaption}
\usepackage{booktabs}
\usepackage{tikz}
\usetikzlibrary{positioning}
\usetikzlibrary{arrows}
\usetikzlibrary{calc,fit}
\usetikzlibrary{shapes.geometric}
\usetikzlibrary{shapes.misc}
\usetikzlibrary{decorations.pathmorphing}
\usetikzlibrary{decorations.pathreplacing}

\newcommand{\eg}{\textit{e.g.}}
\newcommand{\ie}{\textit{i.e.}}

\usepackage[capitalize,noabbrev,nosort]{cleveref}

\theoremstyle{plain}
\newtheorem{theorem}{Theorem}[section]
\newtheorem{proposition}[theorem]{Proposition}

\theoremstyle{definition}

\newtheorem{assumption}[theorem]{Assumption}
\theoremstyle{remark}

\AddToHook{env/proposition/begin}{\crefalias{theorem}{proposition}}
\AddToHook{env/assumption/begin}{\crefalias{theorem}{assumption}}

\makeatletter
\crefname{section}{\S\@gobble}{\S\@gobble}
\crefname{subsection}{\S\@gobble}{\S\@gobble}
\crefname{subsubsection}{\S\@gobble}{\S\@gobble}
\crefname{figure}{Fig.}{Figs.}
\crefname{table}{Table}{Tables}

\def\paragraph{\@startsection{paragraph}{4}{\z@}{0.2em}{-1em}{\normalsize\bf}}
\makeatother

\newcommand{\std}[1]{\scriptsize$\pm$#1}

\DeclareMathOperator*{\argmax}{arg\ max}

\title{Structured Inference with Large Language Gibbs}
\author{
 \textbf{Sanghyeok Choi\textsuperscript{1}}\quad
 \textbf{Henry Gouk\textsuperscript{1}}\quad
 \textbf{Esmeralda S. Whitammer\textsuperscript{1,2}}
\\
\\
 \textsuperscript{1}University of Edinburgh, School of Informatics
 \\
 \textsuperscript{2}CIFAR Fellow, Learning in Machines and Brains
\\
\\
 {
   \tt
   \{sanghyeok.choi, henry.gouk, esmeralda.whitammer\}@ed.ac.uk
 }
}

\begin{document}
\maketitle

\begin{abstract}
\looseness=-1
  The knowledge encoded in large language models (LLMs) can serve as a substrate for structured reasoning over variables describing a complex world, but accessing this knowledge in a probabilistically coherent manner poses a difficult inference problem.
  We propose Large Language Gibbs, a scheme for structured probabilistic inference that uses conditional distributions of an LLM as transition operators.
  Rather than sampling structured objects through single-pass autoregressive generation, we iteratively resample individual variables conditioned on others using an LLM's next-token conditionals.
  This approach avoids order-dependent biases and produces a stationary distribution that reflects a compromise between all local conditionals.
  We apply this approach to sampling from synthetic distributions, consistent reasoning tasks, and Bayesian structure learning.
  The results suggest that the use of LLM conditionals in MCMC is a practical alternative to one-pass generation for structured probabilistic inference under a world prior accessible through noisy LLM conditionals.
  
  Code: \url{https://github.com/hyeok9855/large-language-gibbs}.
\end{abstract}

\section{Introduction}

\begin{figure*}[h]
  \centering
  \newcommand{\metadatacontents}{
    \begin{tikzpicture}
      \node[anchor=west] (x1) at (0,-0) {Breed\vphantom{pdl()}};
      \node[anchor=west] (x2) at (0,-0.5) {Coat\vphantom{pdl()}};
      \node[anchor=west] (x3) at (0,-1) {Age (years)\vphantom{pdl()}};
      \node[anchor=west] (x4) at (0,-1.5) {Sleep (h/day)\vphantom{pdl()}};
      \node[anchor=west] (m1) at (3,-0) {\color{blue}enum(...)\vphantom{pdl()}};
      \node[anchor=west] (m2) at (3,-0.5) {\color{blue}enum(...)\vphantom{pdl()}};
      \node[anchor=west] (m3) at (3,-1) {\color{blue}int[0,30]\vphantom{pdl()}};
      \node[anchor=west] (m4) at (3,-1.5) {\color{blue}int[0,24]\vphantom{pdl()}};
    \end{tikzpicture}
  }

  \newcommand{\samplecontents}[4]{
    \begin{tikzpicture}
      \node[anchor=west] (y1) at (0,0) {#1\vphantom{pdl()}};
      \node[anchor=west] (y2) at (0,-0.5) {#2\vphantom{pdl()}};
      \node[anchor=west] (y3) at (0,-1) {#3\vphantom{p}};
      \node[anchor=west] (y4) at (0,-1.5) {#4\vphantom{p}};
    \end{tikzpicture}
  }
  \includestandalone[width=\linewidth]{figures/gibbs-cat}\\[0.5em]
  \includestandalone[width=\linewidth]{figures/gibbs-cr}
  \caption{\textbf{Above:} Several iterations of large language Gibbs sampling for four variables jointly describing a cat. The initial values are sampled autoregressively, then updates are performed by resampling one variable at a time from the LLM's conditional distribution given the others, serialised in random order (\cref{sec:llg_basic}). Metadata describing the variables can be used to impose constraints on the resampling. \textbf{Below:} Application to consistent reasoning, where each variable is a binary answer to a question (\cref{sec:consistent_reasoning}). Resampling encourages consistency among the answers. For example, the LM's knowledge that `Polaris' and `North Star' refer to the same star may be being imposed in the first resampling step and the knowledge that at most one star is the brightest in the second step.}
  \label{fig:gibbs-examples}
\end{figure*}
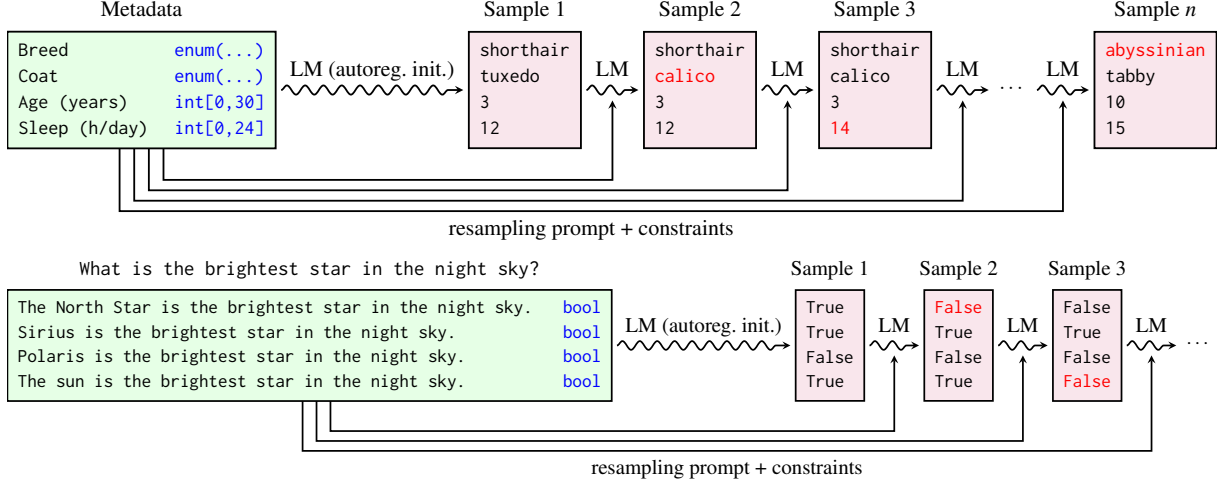

Large language models (LLMs) are increasingly used to represent distributions over structured objects, such as label assignments over large input sets \citep{gadetsky2025large,wen2025unsupervised}, tabular data records \citep{borisov2023language}, and other combinatorial objects.
In these settings, where the target sample space is high-dimensional and the joint distribution is constrained by statistical dependencies among variables, the \emph{estimation of the joint distribution} matters as much as the quality of \emph{marginal predictions on a single variable}, which is the focus of most LLM applications.
A common approach is to generate the full object in a single left-to-right pass, but sequential decoding can introduce order-dependent artifacts such as recency effects \citep{lu-etal-2022-fantastically} and ignorance of context \citep{liu-etal-2024-lost}, which may distort the intended joint distribution.

\looseness=-1
In this work, we ask whether using the conditional beliefs of LLMs over individual variables allows extracting more information than autoregressive generation permits.
Concretely, we study Gibbs-style sampling procedures that iteratively resample one variable conditioned on the current values of the others (\cref{fig:gibbs-examples}), motivated by recent work on the use of masked language models as Gibbs transition operators \citep{wang-cho-2019-bert,yamakoshi-etal-2022-probing,torroba-hennigen-kim-2023-deriving} and the prior use of Gibbs sampling for simple structured prediction tasks with LLMs \citep{zhu2024recovering,zhu2025eliciting} and inference of reasoning chains \citep{xu2024reprompting}; see related work in \cref{app:related_work}.
This procedure turns the model into a transition operator over structured states whose stationary distribution reflects a compromise among local conditionals, removing the bias of any fixed generation order.
We propose several ways to define the transition operator using the LLM's next-token conditionals (\cref{sec:llg_basic,sec:discriminative_kernels}) and study their interaction with the chosen model.

\looseness=-1
We demonstrate this \textbf{large language Gibbs} approach on three illustrative applications.
First, for synthetic distributions, we show that sequential prompting can fail to produce conditionally independent and marginally unbiased samples, consistent with recent observations \cite{hopkins2023can,hu2024amortizing,gu2026illusion}, while Gibbs updates can substantially correct this behaviour (\cref{sec:sampling}).
Second, for unsupervised consistent reasoning on benchmarks such as GSM8K and TruthfulQA, we use iterative conditional updates to find coherent answers across sets of related questions (\cref{sec:consistent_reasoning}) and find that this offers a more principled and better-performing alternative to recent algorithms \citep{wen2025unsupervised}.
Third, for Bayesian structure learning from tabular data, we use Gibbs sampling to generate synthetic data and define a world-knowledge-driven prior over causal structures, finding that this can improve structure learning performance on real datasets (\cref{sec:structure_learning}).

\section{Background: Gibbs Sampling}
\label{sec:gibbs}

Consider a random variable $X=(X_1,\dots,X_n)$ taking values in $\mathcal{X}=\mathcal{X}_1\times\dots\times\mathcal{X}_n$, where each $\mathcal{X}_i$ is a finite (or perhaps infinite discrete) space.

It is well-known \citep{geman1984stochastic,mackay2003information} that if the joint probability mass function $p$ of $X$ has full support, then it is the unique stationary distribution of the Markov chain $\left(X^{(t)}\right)_{t=1}^\infty$ defined by the Gibbs kernel, whose transition rule is given by
\begin{equation}\label{eq:gibbs}
\begin{split}
  i&\sim {\rm Unif}(\{1,\dots,n\}),\\
  X_i^{(t)}&\sim p(X_i^{(t)}\mid X_{-i}^{(t-1)}),
  \end{split}
\end{equation}
where $X_{-i}$ indicates that $i$-th variable is omitted from $X$. That is, the kernel randomly resamples one of the variables from its conditional given all other variables. Variants may select $i$ non-uniformly (for example, introducing temporal dependence by resampling variables in a fixed order) or resample multiple variables at a time, known as block Gibbs sampling. See \cref{sec:practical_variants}.

Suppose now that one has a collection of \emph{approximations} $q_i(X_i\mid X_{-i})$ to the unary conditionals of $p$, and define a Markov chain by the resampling operator
\begin{equation}\label{eq:approx_gibbs}
\begin{split}
  i&\sim {\rm Unif}(\{1,\dots,n\}),\\
  X_i^{(t)}&\sim q_i(X_i^{(t)}\mid X_{-i}^{(t-1)}).
\end{split}
\end{equation}
In contrast to \eqref{eq:gibbs}, the estimate $q_i$ is used for resampling instead of the true conditional of $p$. If each $q_i$ has full support, then the kernel is still given by a stochastic matrix whose $n$-th power is positive, so it has a unique stationary distribution, to which iterative resampling using \eqref{eq:approx_gibbs} converges, even if the unary conditionals $q_i$ are not compatible with \emph{any} joint distribution. Computing the density of this stationary distribution, or its marginals, amounts to a $\prod_i|{\cal X}_i|$-dimensional eigenvector computation, but sampling from it can be easily approximated simply by simulating the Markov chain.

Note also that the stationary distribution of \eqref{eq:approx_gibbs} is only one notion among many of a joint distribution that is optimally compatible with a collection of unary conditionals that do not necessarily arise from the same joint. For example, the Arnold-Gokhale construction \citep{arnold1998distributions} reconstructs a joint distribution from conditionals $q_i$ as the one whose unary conditionals are KL-nearest to the $q_i$.

\section{Large Language Gibbs}

We consider approximate Gibbs kernels \eqref{eq:approx_gibbs} in which the $q_i$ are defined by an autoregressive (large) language model. Assume that we have a fixed language model $p^{\text{LM}}$, with next-token distribution $p^{\text{LM}}(w_j\mid w_{1:j-1})$ for any token sequence $w_{1:j-1}$ in a fixed vocabulary. This language model defines a conditional distribution over continuations of any token sequence. 

We are interested in the joint distribution over variables $X_1,\dots,X_n$ that $p^{\text{LM}}$ models. Each variable $X_i$ is assumed to be associated with a natural language context $c_i$ that describes it, and its sample space $\mathcal{X}_i$ is considered to be in bijection with finite-length sequences of tokens under specific constraints (for example, $X_i$ may be an integer in a certain range, which can be represented by numeric tokens with appropriate constraints, or may be forced to belong to a certain set of strings). We assume that the variables possess no inherent semantic ordering.

The main question is how to define a resampling kernel $q_i^{\text{LM}}(X_i\mid X_{-i})$ using the language model, which only gives the conditional distribution over next token(s) given a previous token sequence.

\subsection{Large Language Gibbs: Basic Version}
\label{sec:llg_basic}

We can define a kernel as the expectation of autoregressive kernels under random permutations,
\begin{equation}
  \label{eq:lm_conditional}
  \begin{split}
  &q^{\text{LM}}_i(X_i \mid X_{-i})
  \\&\triangleq \mathbb{E}_{\sigma_{-i}}\left[ p^{\text{LM}}(X_i \mid \texttt{seq}(X_{-i},\sigma_{-i}) \oplus [c_i])\right],
    \end{split}
\end{equation}
where $\sigma_{-i}$ is drawn uniformly from the set of all bijections $\{1,\dots,n-1\}\to\{1,\dots,n\}\setminus\{i\}$, $\oplus$ denotes sequence concatenation, and $\texttt{seq}(X,\sigma)$ serialises the variables in $X$ into a token sequence $[c_{\sigma(1)}, X_{\sigma(1)},\ldots,c_{\sigma(n-1)}, X_{\sigma(n-1)}]$ in the order $\sigma$ (with appropriate separation tokens), possibly prepended with some global task metadata or an instruction prompt.

For example, at the first step of resampling in the example at the top of \cref{fig:gibbs-examples}, the variable index to be resampled is $i=2$ (``Coat''). With a certain $\sigma_{-2}$, we may resample $X_2$ by continuing the following prompt:
\begin{quote}
  \vspace{-5pt}
  \tt\small
  Here are some characteristics of a cat.

  Sleep (h/day, int[0,24]): 12\\
  Breed (enum(...)): shorthair\\
  Age (years, int[0,30]): 3\\
  Coat (enum(...)):
  \vspace{-5pt}
\end{quote}
Note that the variables are shuffled. Practically, to enforce constraints, we can use ``structured output'' generation in modern language modelling software, \eg, vLLM \citep{kwon2023efficient}.

\paragraph{Algorithm.} Unravelling the approximate Gibbs procedure with the kernel described above, the basic large language Gibbs procedure is the following:
\begin{enumerate}[left=0pt,nosep,label=(\arabic*)]
  \item Initialise variables $X_1,\ldots,X_n$ (\eg, using autoregressive generation with $p^{\text{LM}}$).
  \item For $t=1, \dots, T$:
    \begin{enumerate}[left=0pt,nosep,label=(\alph*)]
      \item Sample an index $i$ uniformly from $1, \ldots, n$.
      \item Sample a bijection $\sigma_{-i}$ uniformly from the set of all bijections $\{1, \ldots, n-1\} \to \{1, \ldots, n\} \setminus \{i\}$.
      \item Sample $X_i$ from\\$p^{\text{LM}}(X_i \mid \texttt{seq}(X_{-i},\sigma_{-i}) \oplus [c_i])$.
      \item Yield current value of $X$ as $X^{(t)}$.
    \end{enumerate}
\end{enumerate}
After simulating the Gibbs chain, we collect a certain number of samples using typical MCMC procedures, such as burn-in (forgetting a fixed number of samples from the start of the chain) and thinning (outputting every $K$-th element of the chain).

\paragraph{On the stationary distribution.}

Because the transition kernels have full support, the large language Gibbs chain has a unique stationary distribution, which we denote by $q^*$. As mentioned before, the conditionals defined in \eqref{eq:lm_conditional} are not necessarily compatible with any joint, and in general $q_i^{\text{LM}}(X_i \mid X_{-i}) \neq q^*(X_i \mid X_{-i})$. Below, we construct a natural candidate joint $q^{\text{sym}}$ and identify a sufficient condition under which $q^* = q^{\text{sym}}$.

First, consider a candidate joint distribution
\begin{equation}
  q^{\text{sym}}(X) \triangleq \mathbb{E}_{\sigma}\left[ p^{\text{LM}}(\texttt{seq}(X,\sigma))\right], \label{eq:qsym_def}
\end{equation}
where $\sigma$ is uniform on the set of all bijections from $\{1,\ldots,n\}$ to itself. This is a natural choice when no variable ordering is preferred over others, and since each $p^{\text{LM}}(\texttt{seq}(\cdot,\sigma))$ is a probability density, so is $q^{\text{sym}}$. Operationally, $q^{\text{sym}}$ is the distribution of single-pass autoregressive generation in a random variable order.

A bijection $\sigma_{-i}$ in \eqref{eq:lm_conditional} corresponds to a $\sigma$ in \eqref{eq:qsym_def} satisfying $\sigma(n)=i$. Using this, we fix the last-position index to $i$ and define
\begin{equation}
  q^{\text{sym}}_i(X) \triangleq \mathbb{E}_{\sigma\,:\,\sigma(n)=i}\left[p^{\text{LM}}(\texttt{seq}(X,\sigma))\right], \label{eq:qsym_i_def}
\end{equation}
so $q^{\text{sym}}(X) = \tfrac{1}{n}\sum_{j=1}^{n} q^{\text{sym}}_j(X)$.

We now ask when $q^* = q^{\text{sym}}$. It suffices that the resampling conditionals \eqref{eq:lm_conditional} coincide with the unary conditionals of $q^{\text{sym}}$, which we establish via the chain
\begin{equation}
  \!q_i^{\text{LM}}(\cdot\! \mid X_{-i}) \overset{\text{(a)}}{=} q^{\text{sym}}_i(\cdot\! \mid X_{-i}) \overset{\text{(b)}}{=} q^{\text{sym}}(\cdot\! \mid X_{-i}),\!
  \label{eq:proof_chain}
\end{equation}
which holds under the following assumption.

\begin{assumption}\label{assumption:invariance}
For all $i\in\{1,\ldots,n\}$ and $X$:
\begin{enumerate}[nosep,label=(\alph*)]
\item $p^{\text{LM}}(\texttt{seq}(X_{-i},\sigma_{-i})\oplus[c_i])$
  does not depend on $\sigma_{-i}$;
\item $q^{\text{sym}}_i(X) = q^{\text{sym}}(X)$.
\end{enumerate}
\end{assumption}
Specifically, (a) states that the LM's likelihood of the serialised context is
invariant to the order of the conditioning variables $X_{-i}$; (b) states that the joint likelihood of the variables as estimated by the LM is invariant to the last variable in the order in which they are serialised.

\begin{proposition}
  \label{prop:stationary}
  Under \cref{assumption:invariance}, the conditionals \eqref{eq:lm_conditional} coincide with the unary conditionals of $q^{\text{\rm sym}}$, \ie,
  \begin{equation*}
    q_i^{\text{\rm LM}}(X_i \mid X_{-i}) = q^{\text{\rm sym}}(X_i \mid X_{-i}), \qquad \forall\, i,\,X.
  \end{equation*}
  Consequently, $q^{\text{sym}}$ is invariant under the large language Gibbs kernel; if in addition each $q^{\text{\rm LM}}_i(\cdot \mid X_{-i})$ has full support on $\mathcal{X}_i$, the chain is irreducible and $q^* = q^{\text{\rm sym}}$.
\end{proposition}

\begin{proof}
  By (a) of \cref{assumption:invariance}, for all $i$, we have
  \begin{equation}
  \begin{split}
  &q_{i}^{\text{LM}}(X_i \!\mid\! X_{-i})
  \\&\propto \mathbb{E}_{\sigma_{-i}}\left[p^{\text{LM}}(\texttt{seq}(X_{-i},\sigma_{-i})\oplus [c_i,X_i])\right]
  \\&=q^{\text{sym}}_i(X).
  \end{split}
  \end{equation}
  The proportionality is in $X_i$; both sides being normalised in $X_i$, $q_i^{\text{LM}}(X_i \mid X_{-i})=q^{\text{sym}}_i(X_i \mid X_{-i})$. Now, since $q^{\text{sym}}_i(X) = q^{\text{sym}}(X)$ under (b), $q^{\text{sym}}_i(X_i \!\mid\! X_{-i}) = q^{\text{sym}}(X_i \!\mid\! X_{-i})=q^{\text{LM}}_i(X_i \!\mid\! X_{-i})$. Compatibility of the unary conditionals with $q^{\text{sym}}$ then makes $q^{\text{sym}}$ stationary by the standard Gibbs argument; positivity gives irreducibility, hence uniqueness, \ie, $q^* = q^{\text{sym}}$.
\end{proof}

We emphasise that, in general, \cref{assumption:invariance} does not hold, and $q^*$ and $q^{\text{sym}}$ are therefore two different distributions. As noted above, $q^{\text{sym}}$ is simply the distribution of autoregressive generation under a random variable ordering and can therefore be sampled from directly. The stationary distribution $q^*$, in contrast, is defined only implicitly by the Gibbs procedure, which can be seen as a compromise among all the LM's unary conditionals. The central idea of this work is to use $q^*$ to elicit the LM's beliefs, under the hypothesis that an autoregressive LM encodes its knowledge more faithfully in its conditional distributions than in any single generation pass. Our empirical results in \cref{sec:applications} support this hypothesis, with $q^{\text{sym}}$ serving as the main baseline (AR-Random).

\begin{figure*}
  \centering
  \captionsetup[subfigure]{labelformat=empty}
  \begin{subfigure}[b]{0.8\linewidth}
    \includegraphics[width=1.0\linewidth]{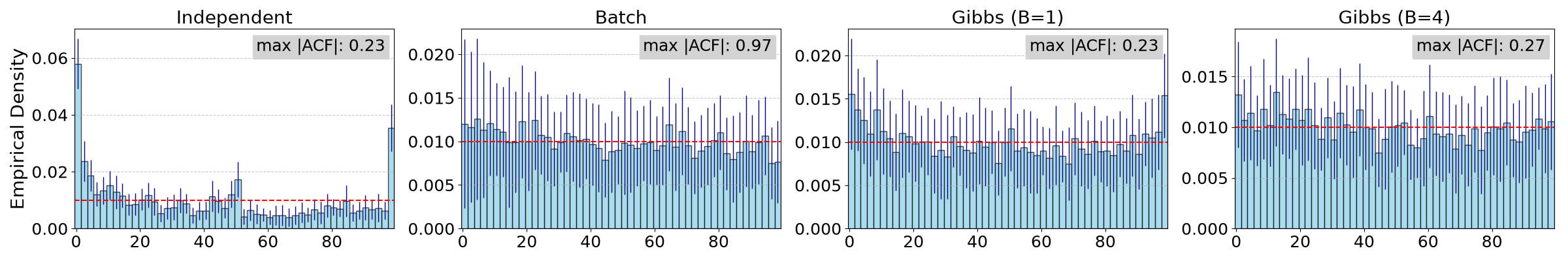}
  \end{subfigure}
  \begin{subfigure}[b]{0.8\linewidth}
    \includegraphics[width=1.0\linewidth]{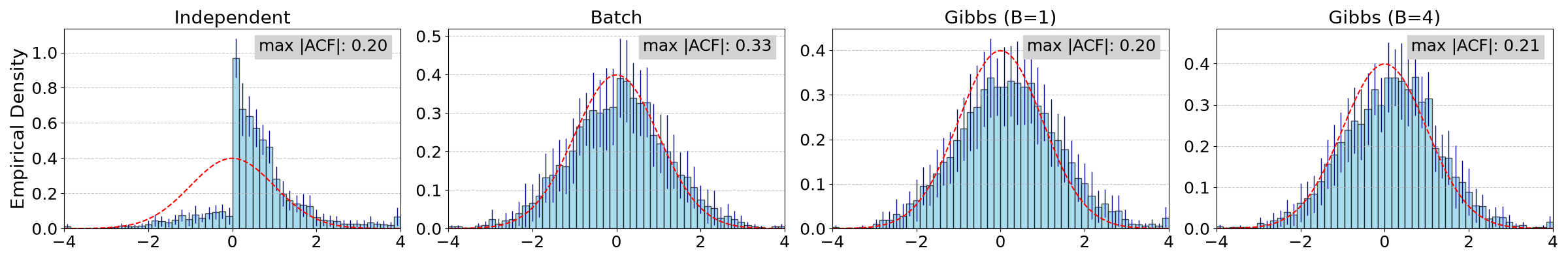}
  \end{subfigure}
  \vspace*{-0.5em}
  \caption{Empirical distribution of generated samples from \textbf{Llama-3.1-8B (base)} for target distribution as $U[\{0,1,\ldots,99\}]$ (top) and $\mathcal{N}(0, 1)$ (bottom). Results are aggregated over 25 seeds, with 256 samples generated per seed. Maximum autocorrelation over 1-128 lags and over 25 runs is also reported ($\max|\text{ACF}|$). Independent sampling has high bias; batch sampling has high autocorrelation. Gibbs substantially corrects both issues.} 
  \label{fig:sampling_base}
\end{figure*}

\begin{figure*}
  \centering
  \captionsetup[subfigure]{labelformat=empty}
  \begin{subfigure}[b]{1.0\linewidth}
    \includegraphics[width=1.0\linewidth]{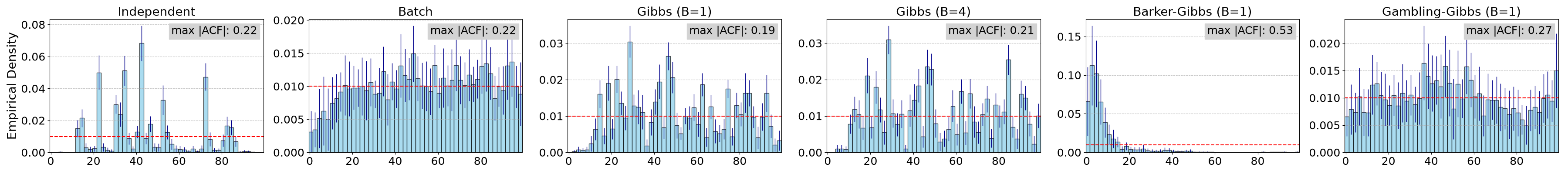}
  \end{subfigure}
  \begin{subfigure}[b]{1.0\linewidth}
    \includegraphics[width=1.0\linewidth]{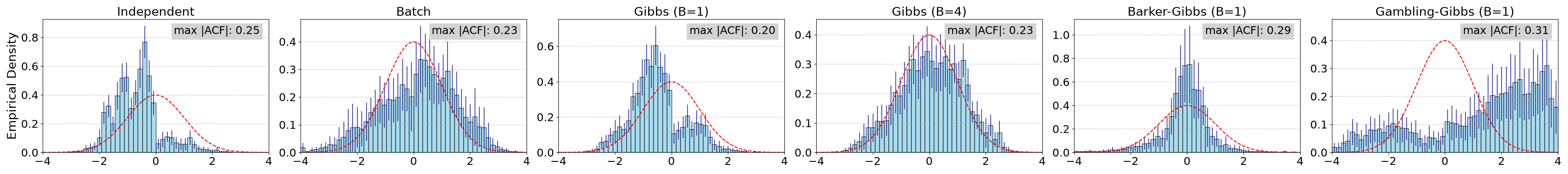}
  \end{subfigure}
  \vspace*{-1.5em}
  \caption{Empirical distribution of generated samples from \textbf{Llama-3.1-8B-Instruct}, same format as \cref{fig:sampling_base}.}
  \label{fig:sampling_instruction_tuned}
  \vspace*{-0.5em}
\end{figure*}

\subsection{Kernels with Discriminative Prediction}
\label{sec:discriminative_kernels}
The basic large language Gibbs procedure in \cref{sec:llg_basic} directly uses the language model's \emph{generative} conditional $p^{\text{LM}}(X_i \mid \texttt{seq}(X_{-i},\sigma_{-i}) \oplus [c_i])$ as the resampling kernel. Depending on the language model type (\eg, instruction-tuned vs. base model, see \cref{sec:sampling} for motivating examples), the LM's discriminative prediction may be more reliable than generation, as reinforcement learning fine-tuning can lead to miscalibrated probabilities \citep{swamy2026all}. This motivates Metropolis-within-Gibbs kernels in which a candidate is drawn from a simple proposal (\eg, uniform) and accepted via the LM's discriminative output.

The algorithm is exactly the same as before (\cref{sec:llg_basic}) except that we replace the resampling step (2c) with the following:
\begin{enumerate}[nosep,label=(\roman*),left=0pt]
  \item Draw a candidate $X_i' \sim q(\cdot)$ from a proposal $q$ on $\mathcal{X}_i$ (taken to be uniform throughout).
  \item Accept or reject $X_i'$ by querying the language model.
\end{enumerate}
We consider two ways to implement the acceptance step (ii).

\paragraph{Barker Gibbs.} Present the language model with the current value $X_i$ and the candidate value $X_i'$, and ask it to \emph{choose the one that is more plausible} given $X_{-i}$ (serialised in a random order). Let $p^{\text{LM}}\left(X_i' \succ X_i \mid X_{-i}\right)$ be the probability of $X_i'$ being preferred by the language model. If we assume that the language model is calibrated in the sense of Barker's rule \citep{barker1965monte}, \ie,
\begin{equation}
\begin{split}
  &p^{\text{LM}}\left(X_i' \succ X_i \mid X_{-i}\right) \\&= \frac{q^*(X_{-i}, X_i')}{q^*(X_{-i}, X_i) + q^*(X_{-i}, X_i')},
  \end{split}\label{eq:barker_rule}
\end{equation}
where $q^*$ is the stationary distribution of the basic large language Gibbs, then the algorithm (combined with the uniform proposal) also has $q^*$ as the stationary distribution. The usage of Barker's rule is also proposed for MCMC with LLMs in \citet{zhu2024recovering}, and we extend it by combining it with the Gibbs sampling framework.

\paragraph{Gambling Gibbs.} Here, we cast the decision of accepting $X_i'$ or not as a gamble. Along with the random candidate $X_i'$, we sample $V \sim U[0,100]$ and ask the language model \emph{whether it would place a bet} $\$V$ that $X_i'$ is more plausible than $X_i$, which will pay out $\$100$ if correct. If the model places a bet, then we accept $X_i'$. Again, if we assume the language model's internal belief of plausibility of $X_i$ is proportional to $q^*(X_{-i}, X_i)$, it should place the bet only if $\frac{V}{100} < \frac{q^*(X_{-i}, X_i')}{q^*(X_{-i}, X_i) + q^*(X_{-i}, X_i')}$, \ie,
\begin{equation}
  \begin{split}
    &p^{\text{LM}}\left(X_i' \succ X_i \mid X_{-i}\right) \\
    &= \mathbb{P}\left(\frac{V}{100} < \frac{q^*(X_{-i}, X_i')}{q^*(X_{-i}, X_i) + q^*(X_{-i}, X_i')}\right),
  \end{split}
\end{equation}
which recovers the Barker acceptance probability \eqref{eq:barker_rule}. The advantage of Gambling Gibbs over Barker Gibbs is that the language model is required only to provide a deterministic threshold judgment, which can be obtained by greedy decoding and does not rely upon calibrated probabilities.

\subsection{Practical Variants}
\label{sec:practical_variants}

Variations on the basic Gibbs sampling from \cref{sec:gibbs} are possible. We describe two that we consider in combination with the large language Gibbs kernels.

\paragraph{Sweeping Gibbs.} Instead of randomly selecting a variable to resample at each step, we can cycle through the variables in a fixed order, a common practice in Gibbs sampling that can lead to faster convergence. If the resampling kernels are exact conditionals of some target distribution, then this procedure has that target as its stationary distribution. However, in the case of approximate conditionals that are not compatible with any joint, the stationary distribution may be different from that obtained with random variable resampling.

\paragraph{Block Gibbs.} One can resample sets of $B>1$ variables at a time from an approximation to their conditional given the others. The large language Gibbs kernel is readily generalised to this case: one serialises the remaining $n-B$ variables in random order and places them in the context, then autoregressively resamples the $B$ variables to be updated. Such resampling can help the chain escape local modes in the case of highly correlated variables. Similar results to \cref{prop:stationary}, characterising the stationary distribution, can be obtained for the block Gibbs case. The corresponding assumptions generalising \cref{assumption:invariance} require that the joint likelihood symmetrised over all permutations equals the joint likelihood symmetrised only over permutations where the sets of the first $n-B$ and last $B$ indices are fixed.

\section{Motivating Example: Sampling from Simple Distributions}
\label{sec:sampling}

Multiple studies in the literature report that LLMs are not able to sample from simple distributions over integers and real numbers, such as uniform and Gaussian \citep{hopkins2023can,hu2024amortizing,gu2026illusion}. Notably, \citet{gu2026illusion} showed that batched sampling with LLMs (generating a batch of samples sequentially in one shot) gives better marginal distributions than independent sampling (generating one sample at a time). However, the batch samples sometimes have high autocorrelation, meaning that the samples are not independent despite having more accurate marginals.

Here, we use large language Gibbs and its variants to sample from simple distributions: the uniform distribution over integers from 0 to 99, \ie, $U[\{0,1,\ldots,99\}]$, and the standard Gaussian distribution $\mathcal{N}(0, 1)$. We use Llama-3.1-8B \citep{grattafiori2024llama}, both base and instruction-tuned variants, and prompt the language model to generate independent samples from the target distribution with different sampling strategies (independent, batch, and Gibbs sampling). For Gibbs algorithms, we arbitrarily choose the number of variables to be 16, \ie, $X = (X_1, \ldots, X_{16})$, where each $X_i$ is expected to be an \emph{i.i.d.} sample from the target distribution. We also experiment with the block-Gibbs scheme (\cref{sec:practical_variants}), updating $B=4$ variables at a time. For details, see \cref{app:setups:sampling}.

\cref{fig:sampling_base} shows the empirical distribution of samples from the base model with different sampling approaches. Independent sampling has high error and batch sampling has high autocorrelation, confirming the findings in \citet{gu2026illusion}. Gibbs sampling is more faithful than independent sampling and has lower autocorrelation than batch sampling. We also report the results from the OLMo-3 32B models \citep{olmo2025olmo} in \cref{app:results:sampling:olmo} and results on additional target distributions in \cref{app:results:sampling:different_targets}.

\cref{fig:sampling_instruction_tuned} shows the results from an instruction-tuned model. From the empirical distribution of independent sampling, we observe that instruction-tuning introduces some bias, which affects batch and Gibbs sampling results as well. Gambling-Gibbs does not sample from the target distribution and instead seems to accept or reject randomly; we suspect this is because models of this scale are not good at the gambling task with numerical values. As shown in  \cref{fig:gambling_gibbs_reasoning} of \cref{app:results:sampling:improved_gambling_gibbs}, we can partially resolve this issue by using a larger model and chain-of-thought reasoning, which reveals that the model attempts to deduce the correct acceptance rate by computing density ratios (\cref{sec:cots}).

\section{Applications} \label{sec:applications}

We now demonstrate the large language Gibbs approach on two applications: unsupervised consistent reasoning (\cref{sec:consistent_reasoning}) and LLM-derived priors for Bayesian structure learning (\cref{sec:structure_learning}). In both cases, we show that large language Gibbs can substantially improve over methods that do not use iterative resampling.

\begin{table}[t]
  \centering
  \resizebox{1\linewidth}{!}{
    \begin{tabular}{@{}lcccc}
      \toprule
      Task $\rightarrow$ & \multicolumn{2}{c}{TruthfulQA} & \multicolumn{2}{c}{GSM8K-Verification} \\
      \cmidrule(lr){2-3} \cmidrule(lr){4-5}
      Alg. $\downarrow$ LLM $\rightarrow$ & \small Llama-3.1-8B & \small OLMo-3-32B & \small Llama-3.1-8B & \small OLMo-3-32B \\
      \midrule
      Zero-shot (greedy) & 0.588 & 0.610 & 0.533 & 0.604 \\
      \midrule
      AR-Fixed, $n\!=\!4$ & 0.654\std{0.013} & 0.656\std{0.008} & 0.658\std{0.013} & 0.695\std{0.010} \\
      \phantom{AR-Fixed,} $n\!=\!16$ & 0.678\std{0.009} & 0.669\std{0.009} & 0.700\std{0.003} & 0.741\std{0.009} \\
      \phantom{AR-Fixed,} $n\!=\!64$ & 0.714\std{0.003} & 0.696\std{0.004} & 0.757\std{0.012} & 0.783\std{0.003} \\
      \midrule
      AR-Random, $n\!=\!4\!\!\!$ & 0.670\std{0.005} & 0.680\std{0.008} & 0.693\std{0.008} & 0.770\std{0.003} \\
      \phantom{AR-Random,} $n\!=\!16\!\!\!$ & 0.701\std{0.005} & 0.680\std{0.004} & 0.735\std{0.006} & 0.772\std{0.007} \\
      \phantom{AR-Random,} $n\!=\!64\!\!\!$ & 0.720\std{0.010} & 0.706\std{0.013} & 0.792\std{0.011} & 0.803\std{0.006} \\
      \midrule
      ICM, $n\!=\!64$ & 0.701\std{0.010} & 0.709\std{0.014} & 0.699\std{0.020} & 0.787\std{0.016} \\
      \midrule
      Gibbs, $n\!=\!4$ & 0.708\std{0.010} & 0.664\std{0.002} & 0.710\std{0.009} & 0.779\std{0.005} \\
      \phantom{Gibbs,} $n\!=\!16$ & 0.715\std{0.008} & 0.693\std{0.008} & 0.786\std{0.006} & 0.812\std{0.005} \\
      \phantom{Gibbs,} $n\!=\!64$ & \textbf{0.744\std{0.010}} & \textbf{0.776\std{0.006}} & \textbf{0.888\std{0.006}} & \textbf{0.869\std{0.003}} \\
      \bottomrule
    \end{tabular}
  }
  \vspace*{-0.5em}
  \caption{Task accuracies on the consistent reasoning tasks (mean$\pm$std over five seeds). Runtime information is given in \cref{app:results:consistent_reasoning:runtime}.}
  \vspace*{-1em}
  \label{tab:consistent_reasoning}
\end{table}

\subsection{Consistent Reasoning}
\label{sec:consistent_reasoning}

Recent works \citep{gadetsky2025large,wen2025unsupervised} have demonstrated that LLMs are able to perform \emph{unsupervised} in-context learning to consistently reason over a dataset of examples by maximising some objective that aggregates conditional probabilities. The key intuition behind these approaches is that finding a set of answers to a set of queries that are consistent with the language model's internal belief \emph{and with each other} can be useful for downstream tasks. Here, we provide a more principled probabilistic point of view on this problem and a solution using large language Gibbs.

Let $c_1,\dots,c_n$ be a set of queries, such as input questions in a question-answering or mathematical reasoning task, and let $X = (X_1, \ldots, X_n)$ be an aligned set of variables representing their (unknown) answers. Both \citet{gadetsky2025large} and \citet{wen2025unsupervised} define an optimisation problem that maximises the product of language model's conditional probabilities over each variable given the others, \ie,
\begin{equation}\label{eq:weird_joint}
  \argmax_{X} \prod_{i=1}^{n} \tilde{q}_i^{\text{LM}}(X_i \mid X_{-i}),
\end{equation}
where $\tilde{q}_i^{\text{LM}}$ is a conditional distribution of $X_i$ given $X_{-i}$ with a specific way of prompting the language model to answer one question $c_i$ with all other questions $c_{-i}$ and their answers $X_{-i}$ placed in context. Since solving the problem is not trivial, the previous approaches either impose a specific ordering of the variables \citep{gadetsky2025large} or use a simulated-annealing-like heuristic \citep{wen2025unsupervised}. Unfortunately, \emph{the product of conditionals in \eqref{eq:weird_joint} does not naturally correspond to the joint likelihood in any probabilistic model}: even if the $\tilde q_i$ are the conditionals of some joint distribution, the product does not recover this joint.

We argue that large language Gibbs allows a more principled and straightforward approach to this problem. The algorithm is exactly the same as described in \cref{sec:llg_basic}, where the input queries $c_i$ are treated as the variable descriptions (\cref{fig:gibbs-examples}, bottom).

\paragraph{Setup.} We borrow the benchmark tasks from \citet{wen2025unsupervised}, which include TruthfulQA and GSM8K-Verification (see \cref{app:setups:consistent_reasoning} for examples). The goal is to classify a given claim (candidate answer) for a question as true or false using a language model. Since there are 4 claims per question, we define $X$ to have $n=4m$ variables, where $m$ is the number of questions considered together at a time. Each $X_i$ in $X$ is a binary variable indicating whether the associated question-claim pair is true or false. We use $m=1,4,16$ (\ie, $n=4,16,64$). Using a larger $m$ could be helpful if the questions are related to each other, in a similar sense to in-context learning \citep{brown2020language}.

We use two base models: Llama-3.1-8B \citep{grattafiori2024llama} and OLMo-3-1125-32B \citep{olmo2025olmo}. We consider Zero-shot, AR-Fixed and AR-Random (autoregressive generation in a fixed and random order, respectively), and Internal Coherence Maximisation \citep[ICM;][]{wen2025unsupervised} as the baseline algorithms.\footnote{The original ICM proposed in \citet{wen2025unsupervised} also uses formal consistency checks that assume access to ground truth labels, but we remove them for a fair comparison of the probabilistic inference algorithms. See \cref{app:setups:consistent_reasoning} for details.} For the large language Gibbs algorithms, we use the basic version (\cref{sec:llg_basic}). See \cref{app:setups:consistent_reasoning} for more detailed explanations of the overall setting and baseline implementations and \cref{app:results:consistent_reasoning:instruct} for results with instruction-tuned models.

\paragraph{Results.} \cref{tab:consistent_reasoning} shows the results. We find, consistently with \citet{gadetsky2025large,wen2025unsupervised}, that language models can be improved by an iterative procedure that promotes internal consistency, surpassing zero-shot or autoregressive baselines. Large language Gibbs with $n=64$ outperforms ICM with the same context length by a large margin in all settings. The monotonic performance gain from AR-Fixed to AR-Random and then Gibbs confirms the importance of shuffling the variables for better consistency and the benefit of iterative resampling.

\begin{figure*}
  \centering
  \captionsetup[subfigure]{labelformat=empty}
  \begin{subfigure}[b]{0.49\linewidth}
    \includegraphics[width=1.0\linewidth]{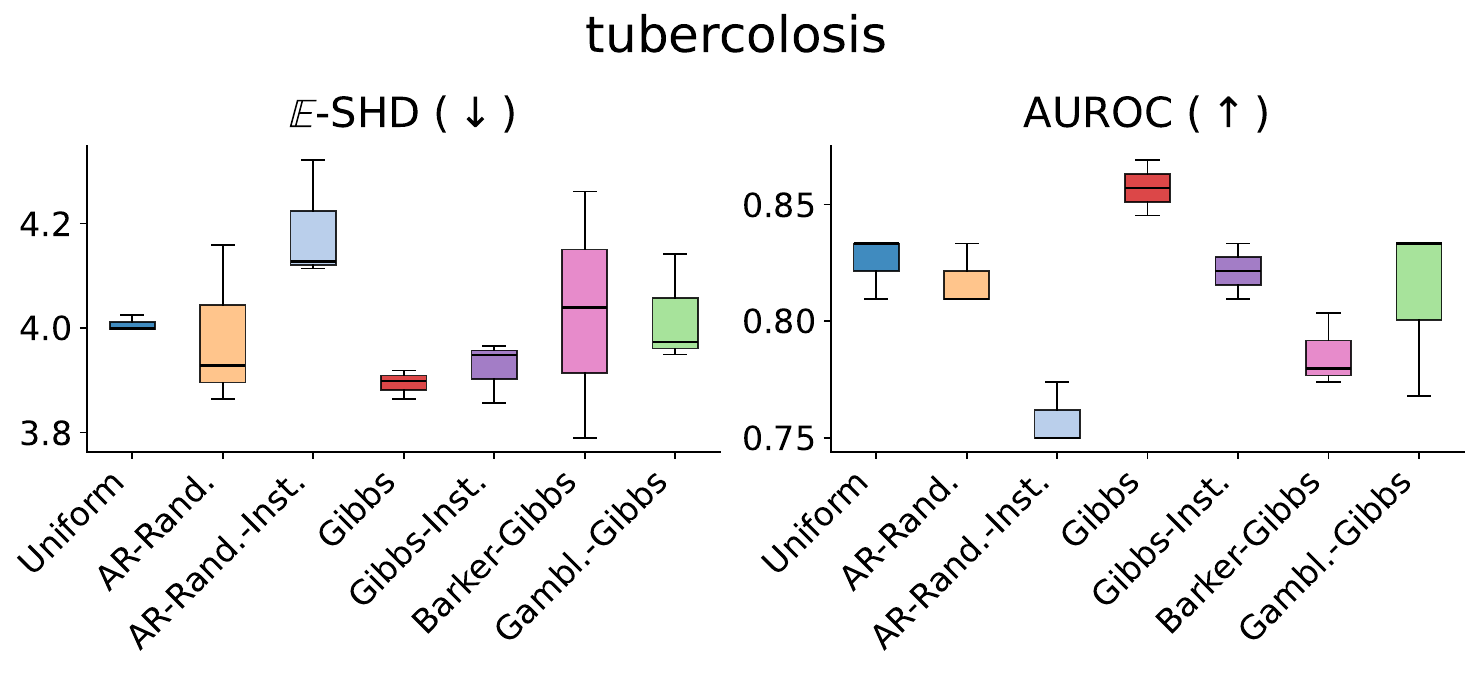}
  \end{subfigure}
  \begin{subfigure}[b]{0.49\linewidth}
    \includegraphics[width=1.0\linewidth]{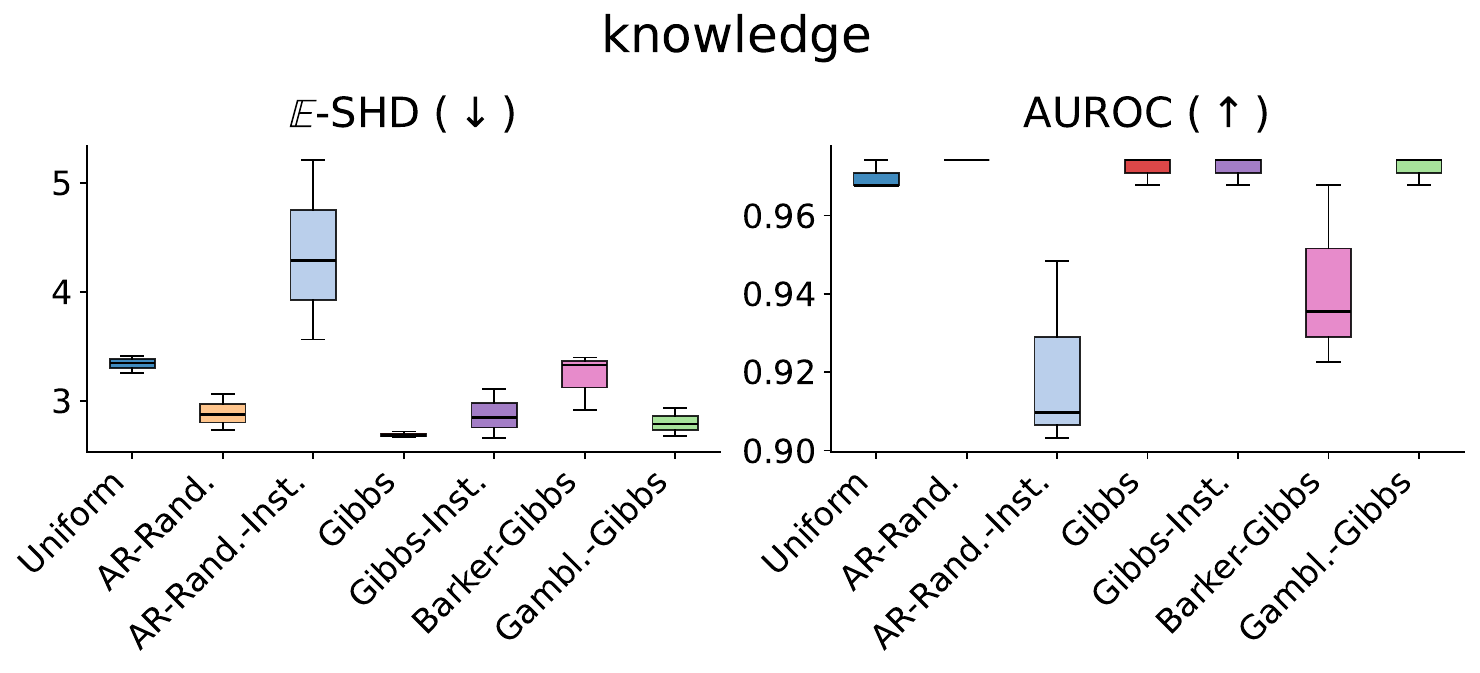}
  \end{subfigure}
  \begin{subfigure}[b]{0.49\linewidth}
    \includegraphics[width=1.0\linewidth]{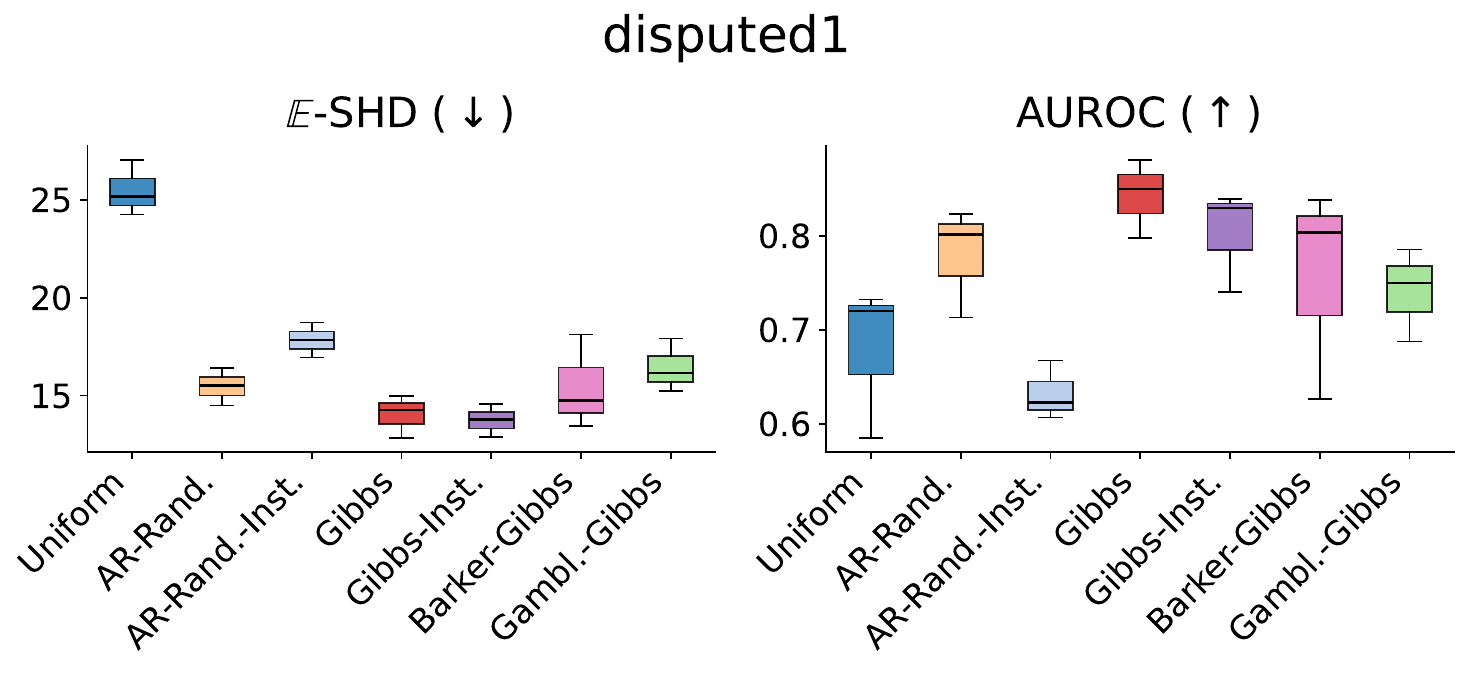}
  \end{subfigure}
  \begin{subfigure}[b]{0.49\linewidth}
    \includegraphics[width=1.0\linewidth]{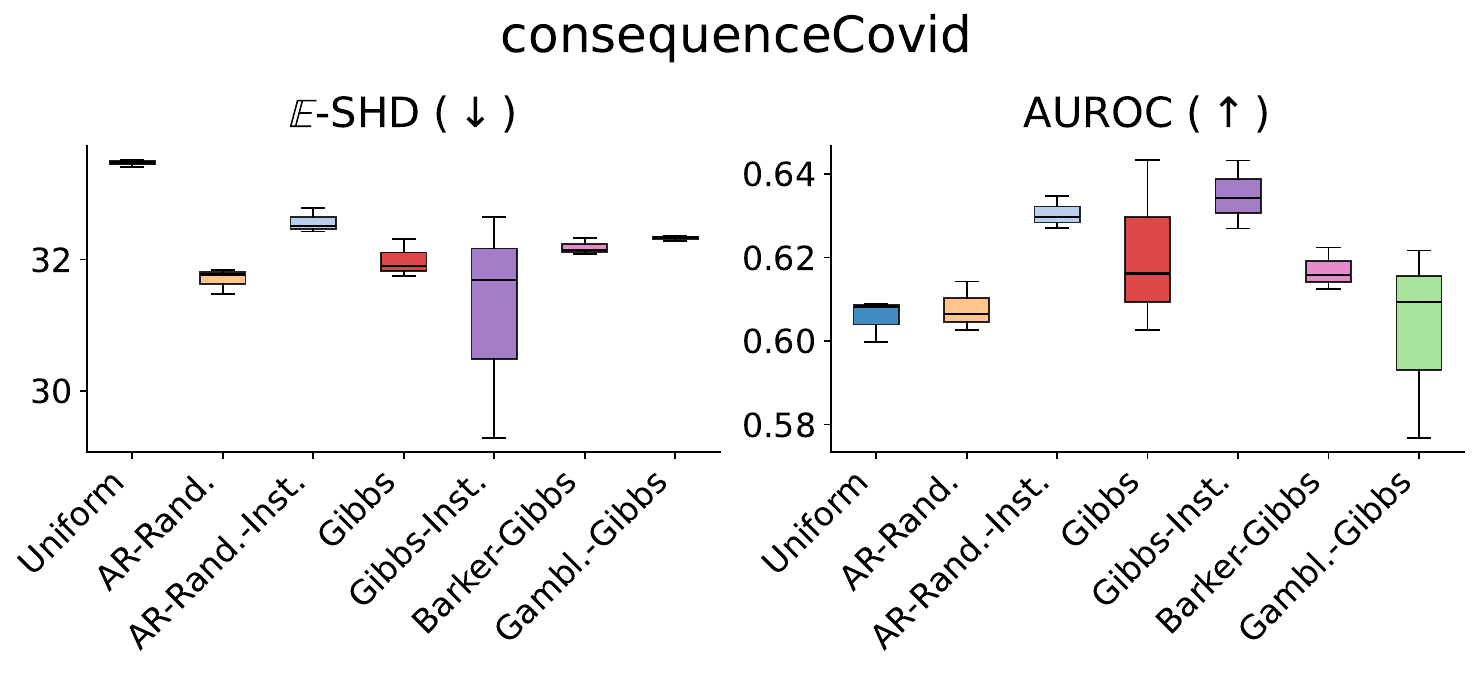}
  \end{subfigure}
  \vspace*{-0.5em}
  \caption{Bayesian structure learning results on four BnRep datasets using Llama-3.1-8B (3 seeds).}
  \label{fig:structure_learning}
  \vspace*{-1em}
\end{figure*}

\subsection{Bayesian Structure Learning}
\label{sec:structure_learning}

Discovering the structure of a directed acyclic graph (DAG) representing a Bayesian network that explains a set of observations is an interesting problem with a prominent application in causal discovery \citep{spirtes,chickering2003optimal}. \emph{Bayesian structure learning} is a Bayesian approach to this problem, where the goal is to approximate or sample from a posterior distribution over DAGs, rather than finding a single graph that best models the data \citep{madigan1995bayesian,giudici2003improving,viinikka2021towards,lorch2022amortized}. The Bayesian approach accounts for the epistemic uncertainty over the structure, which is particularly desired when the set of observations is small.

The Bayesian structure learning problem can be formulated as finding a posterior distribution $P(G \mid D)$ over DAGs $G$, whose vertices correspond to a set of variables, given a set of observations $D$ of joint values of those variables, defined as:
\begin{equation}
  P(G \mid D) \propto P(G)P(D \mid G).
  \label{eq:bayesian_structure_learning}
\end{equation}
Here $P(G)$ is a prior distribution over DAGs and $P(D \mid G)$ is the likelihood of the observations given the graph, with the continuous parameters of the conditional probability distributions in the associated Bayesian network marginalised out. The likelihood is usually defined in such a way that all graphs that encode the same conditional independencies (\ie, are Markov-equivalent) have the same likelihood under certain assumptions \citep{koller2009probabilistic}. An example of such a likelihood for discrete variables is the Bayesian Dirichlet equivalence score \citep[BDe;][]{heckerman1995learning}.

One way to define an LLM-informed prior over $G$ is by using \emph{synthetic} data $D_{\text{LLM}}$ that is generated using large language Gibbs. Given a dataset $D_{\text{LLM}}$ of samples from the joint distribution over the variables considered, we define
\begin{equation}
  P(G \mid D_{\text{LLM}},\gamma) \propto P_{0}(G)P(D_{\text{LLM}} \mid G)^{\gamma},
  \label{eq:llm_prior}
\end{equation}
where $P_{0}(G)$ is the uninformative prior (\eg, uniform over all possible DAGs) and $\gamma$ is a parameter controlling the effective dataset size. $P(D_{\text{LLM}}\mid G)$ can be computed in the same way we compute $P(D \mid G)$. Substituting $P(G\mid D_{\text{LLM}},\gamma)$ for $P(G)$ in \eqref{eq:bayesian_structure_learning} yields:
\begin{equation}
\begin{split}
  &P(G \mid D;D_{\text{LLM}},\gamma) \\&\propto P_{0}(G)P(D_{\text{LLM}} \mid G)^{\gamma}P(D \mid G).
\end{split}
  \label{eq:posterior_with_llm_prior}
\end{equation}
Effectively, \eqref{eq:posterior_with_llm_prior} modifies the Bayesian structure learning problem \eqref{eq:bayesian_structure_learning} by placing upon it a bounded KL constraint, or regulariser, with respect to the posterior given synthetic LLM data. Such a constraint allows us to disambiguate structures that are otherwise difficult to distinguish with a limited dataset $D$ by eliciting the world knowledge contained in the language model.

\paragraph{Setup.} To obtain the synthetic data $D_{\text{LLM}}$ using LLMs, semantic metadata for each variable must be available. We use BnRep \citep{leonelli2025bnrep}, a collection of Bayesian networks drawn from a broad range of literature. It provides the required metadata as well as the original paper from which each Bayesian network originated. Since BnRep was published recently, it is less prone to data contamination than standard benchmarks such as BnLearn \citep{scutari2010learning}. We experiment with four datasets from BnRep: \texttt{tubercolosis}\footnote{The typo in the dataset name is inherited from the original BnRep paper \citep{leonelli2025bnrep}.} ($n=5$), \texttt{knowledge} ($n=6$), \texttt{disputed1} ($n=11$), and \texttt{consequenceCovid} ($n=15$).

The benchmarking procedure is as follows: 
\begin{enumerate}[nosep,label=(\arabic*),left=0pt]
\item generate $D_{\text{LLM}}$ using large language Gibbs (or its variants);
\item run an established Bayesian structure learning algorithm (we use \citet{deleu2022bayesian}) to learn an amortised posterior model using \eqref{eq:posterior_with_llm_prior} as a target distribution;
\item evaluate the posterior against the provided ground truth graph. 
\end{enumerate}
Note that the ground truth graph is unavailable to the algorithm.

We use a \textbf{Uniform} prior as a baseline that does not use a language model and the \textbf{AR-Rand.} (AR-Random) approach as a non-Gibbs baseline, in which a data point with $n$ features is generated autoregressively in a single pass in random feature order. For all algorithms, we use Llama-3.1-8B, where only \textbf{AR-Rand.} and \textbf{Gibbs} use a base model and all other methods use an instruction-tuned model. See \cref{app:setups:structure_learning} for more details.

\paragraph{Results.} In \cref{fig:structure_learning}, we report the expected Structured Hamming Distance ($\mathbb{E}\text{-SHD}$) and the Area Under the Receiver Operating Characteristic (AUROC) between posterior samples and the ground truth graph; see \cref{app:setups:structure_learning} for definitions. Gibbs sampling consistently outperforms \textbf{Uniform}, indicating that large language Gibbs is able to generate useful synthetic data solely from metadata alone. \textbf{AR-Rand.-Inst.} sometimes fails (\eg, \texttt{tubercolosis} and \texttt{knowledge}), which we believe is related to the finding in \cref{sec:sampling} that independent sampling yields biased distributions. We also find that informative metadata is key to success, as illustrated by a failure case in \cref{app:failure_cases}. The appendix includes additional results from OLMo-3-32B (\cref{app:results:structure_learning:olmo}) and study of the effect of $\gamma$ (\cref{app:results:structure_learning:gamma}).

\section{Related Work}
\label{app:related_work}

\paragraph{MCMC with (masked) language models.} A line of work has sought to construct a joint distribution given unary conditionals from masked LMs (\eg, BERT \citep{devlin-etal-2019-bert}), and used MCMC algorithms to sample from it. \citet{wang-cho-2019-bert,goyal2022exposing} viewed a masked LM as a fully-connected Markov random field (MRF) by defining the log-potential of the full graph as a sum of the log probabilities from unary conditionals, and tried to apply Gibbs or Metropolis-Hastings algorithms to sample from the derived joint distribution. \citet{yamakoshi-etal-2022-probing} provided an alternative point of view, interpreting masked LMs as dependency networks \citep{heckerman2000dependency}, and discussed the stationary distribution of the Gibbs procedure. \citet{torroba-hennigen-kim-2023-deriving} extended this perspective and considered algorithms to reconstruct the joint distribution from unary conditionals, and provided an empirical comparison of different approaches to joint construction using a tractably small two-token model.  While these works substantially inspired our work, the fundamental setting differs: these works study the joint distribution over token sequences for language modelling, whereas we perform structured inference over semantically defined variables using autoregressive LMs.

\paragraph{Sampling from a joint distribution of autoregressive LLMs.} Beyond masked language models, MCMC algorithms have been used for constrained generation under autoregressive LM priors in embedding spaces \citep{kumar-etal-2022-gradient,qin2022cold,du2024principled} and token space \citep{miao2019cgmh,mireshghallah-etal-2022-mix}. Metropolis-Hastings algorithms are used for machine translation \citep{faria2024quest} and for sampling (parts of) reasoning chains from a tempered distribution \citep{karan2025reasoning}. \citet{xu2024reprompting} use a Gibbs-sampling-like procedure to find a set of chain-of-thought prompts that works well for a training set, with the hope that it can be reused for the test set. \citet{zhu2024recovering,zhu2025eliciting} study algorithms similar to ours at a smaller scale, based upon which we provide a more general and rigorous framework (\cref{sec:llg_basic}), new kernels (Barker and Gambling Gibbs, \cref{sec:discriminative_kernels}), and novel applications to structure inference (\cref{sec:consistent_reasoning,sec:structure_learning}).

\paragraph{Prior elicitation from LLMs.} Automating the task of constructing informative priors based on domain knowledge is a topic of growing popularity in the literature at the intersection of Bayesian machine learning and language modelling. Several works have investigated using LLMs to select parameter values for families of parametric priors (e.g., Gaussian or Bernoulli) chosen by humans. The resulting informative priors have been shown to be useful for downstream modelling tasks, such as causal discovery, imputation, and predictive modelling \citep{darvariu2024large, capstick2024autoelicit, arai2025many, selby2025had}. Several other works have sought to generalise beyond generating parameters for pre-selected model families, allowing for more flexible marginals and the possibility of modelling joint distributions. This includes works that have used LLMs to specify interpretable parametric models by generating formal descriptions of models \citep{selby2025had, domke2025large}. \citet{gouk2024automated} go even further by abandoning any attempt at interpretability; they use the ``AR-Random'' approach we consider in our analysis, coupled with a separate likelihood for the synthetic data, in the context of predictive modelling with limited data. Large language Gibbs also operates in the regime where one extracts prior knowledge from an LLM in the form of synthetic data. In contrast to previous approaches, we provide empirical evidence that constructions based on Gibbs sampling typically result in beliefs that are more consistent with the underlying data distributions than the AR-Random approach.

\paragraph{Other related works.} The iterative procedure of the proposed Gibbs sampling approach is similar to algorithms proposed in \citet{gadetsky2025large} and \citet{wen2025unsupervised}. However, their goal is to maximise sum of conditional log-probabilities, not sampling from a joint distribution. See \cref{sec:consistent_reasoning} for a detailed comparison. Our Barker Gibbs (\cref{sec:discriminative_kernels}) can be seen as a Gibbs extension of the rejection-sampling approach in \citet{xiao2025flipping}, in which an LLM is asked to accept or reject a candidate sample from a proposal distribution.

\section{Conclusion}

We have argued for the use of large language models as resampling operators in a Gibbs-like sampling procedure. Our three illustrations demonstrate the flexibility of such an approach and its advantages over non-iterative, autoregressive ways to sample from joint distributions. Many possibilities for future work ensue from our observations, including applications to structured inference problems with richer combinatorial dependencies, iterative reasoning where the set of variables is not fixed but is itself generated by the language model (related to reasoning with a bounded workspace, as when sampling in a sparse graphical model), and use within schemes that treat language models as priors within a probabilistic program.

\section*{Limitations}

Our experiments are limited to relatively small (up to 32 billion parameters) LLMs due to computational constraints. 

Large language Gibbs requires more computation to generate samples than simpler approaches such as autoregressive generation due to the iterative procedure, which may require many iterations to converge to the stationary distribution. The burn-in and thinning lengths in our experiments were set heuristically rather than via formal convergence diagnostics such as trace plots or R-hat statistics.

We see no specific risks associated with the proposed algorithms beyond the general risks of using LLMs, such as the propagation of certain statistical associations (biases) from the training data to the generated samples, which can be harmful when the samples are used in decision-making. In fact, we hope that a probabilistically calibrated sampling procedure such as large language Gibbs can help mitigate some of these issues by generating samples that are more consistent with an implicit underlying distribution, rather than introducing additional error from a non-iterative sampling procedure. We also note that the use of LLMs for prior elicitation in Bayesian structure learning has attracted attention in the medical domain (see \cref{app:related_work}), and extra care should be taken in such applications, as well as safety-critical domains in general, to ensure that results are not trusted without proper validation.

\section*{Acknowledgements}

The work of the authors is supported by the Advanced Research and Invention Agency (ARIA). ESW acknowledges support from the CIFAR Learning in Machines and Brains programme.

\bibliography{references}

\appendix

\onecolumn

\section{Experiment Setups}
\label{app:setups}

All experiments use structured output generation to enforce some required format (JSON, regex, or choice). We always decode with temperature 1.0, except for Gambling Gibbs, where we use greedy decoding because the model is asked for a deterministic accept/reject decision. We use sweeping by default.

\subsection{Sampling from Simple Distributions (\cref{sec:sampling})}
\label{app:setups:sampling}

For each run, we generate 256 samples from the target distribution. Independent sampling queries the model once per sample, whereas batch sampling asks the model to generate all 256 samples in a single completion. For the Gibbs-based methods, we define a joint random variable that contains 16 iid variables, $X=(X_1,\ldots,X_{16})$. We consider the basic Gibbs (\cref{sec:gibbs}) and the block Gibbs kernel with $B=4$ (\cref{sec:practical_variants}). Also, Barker and Gambling Gibbs (\cref{sec:discriminative_kernels}) with $B=1$ are considered only for instruction-tuned  LLMs. The total Gibbs chain length is $\texttt{burn\_in} + (256/16) \times \texttt{thinning}$, where \texttt{burn\_in} is $100$ and \texttt{thinning} is $32 / B$. On an NVIDIA H100 80GB GPU, each experiment (25 seeds in parallel) runs in less than 10 minutes for Llama-3.1-8B models.

\subsubsection{Prompts Examples}

Here we provide the prompts we used for the experiments in \cref{sec:sampling}. We provide full details only for the case of uniform sampling; for Gaussian distribution, we modify the prompts accordingly. For all instruction-tuned models, we used the following system prompt: \texttt{\small You are a helpful assistant that can sample from probability distributions.}

For the \textbf{independent} sampling with \textbf{base models}, we generate a sample at a time, 256 times, using the following prompt:
\begin{quote}
    {\setlength{\parindent}{0cm}
    \ttfamily \small
    Here is a random sample from a uniform distribution over the integers in \{0, 1, ..., 99\}, formatted as JSON:
  }
\end{quote}

For the \textbf{independent} sampling with \textbf{instruction-tuned models}, we generate a sample at a time, 256 times, using the following prompt:
\begin{quote}
    {\setlength{\parindent}{0cm}
    \ttfamily \small
    Draw a random sample from a uniform distribution over the integers in \{0, 1, ..., 99\}. Respond with JSON that follows this schema: \{"type": "object", "properties": \{"sample": \{"type": "integer", "minimum": 0, "maximum": 99\}\}, "required": ["sample"]\}
  }
\end{quote}

For the \textbf{batch} sampling with \textbf{base models}, we generate a batch of 256 samples at once using the following prompt:
\begin{quote}
    {\setlength{\parindent}{0cm}
    \ttfamily \small
  Here are 256 iid samples from a uniform distribution over the integers in \{0, 1, ..., 99\}, formatted as JSON:
  }
\end{quote}

For the \textbf{batch} sampling with \textbf{instruction-tuned models}:
\begin{quote}
    {\setlength{\parindent}{0cm}
    \ttfamily \small
    Draw 256 iid samples from a uniform distribution over the integers in \{0, 1, ..., 99\}. Respond with JSON that follows this schema:
    \{"type": "object", "properties": \{"samples": \{"type": "array", "items": \{"type": "integer", "minimum": 0, "maximum": 99\}, "minItems": 256, "maxItems": 256\}\}, "required": ["samples"]\}
  }
\end{quote}

For the \textbf{Gibbs} sampling, we shuffle the variables $X_1, \ldots, X_{16}$ and update last $B$ variables at each step. With \textbf{base models}, the following prompt ($B=4$) is used:
\begin{quote}
    {\setlength{\parindent}{0cm}
    \ttfamily \small
    Here are 12 iid samples from a uniform distribution over the integers in \{0, 1, ..., 99\}:\\
    \{"X13": 65, "X2": 9, ..., "X8": 72\}\\
    Here is another set of 4 iid samples from the same distribution, formatted as JSON:
  }
\end{quote}

\textbf{Gibbs} sampling with \textbf{instruction-tuned models} uses the following prompt ($B=4$):
\begin{quote}
    {\setlength{\parindent}{0cm}
    \ttfamily \small
    You are generating 16 iid samples from a uniform distribution over the integers in \{0, 1, ..., 99\}. You have already drawn 12 iid samples:\\\
    \{"X13": 65, "X2": 9, ..., "X8": 72\}\\
    Draw another set of 4 iid samples from the same distribution. Respond with JSON that follows this schema: \{"type": "object", "properties": \{"X5": \{"type": "integer", "minimum": 0, "maximum": 99\}, ..., "X11": \{"type": "integer", "minimum": 0, "maximum": 99\}\}, "required": ["X5", "X9", "X14", "X11"]\}
  }
\end{quote}

\textbf{Barker Gibbs} sampling (instruction-tuned models only) uses the following prompt:
\begin{quote}
    {\setlength{\parindent}{0cm}
    \ttfamily \small
    You are generating 16 iid samples from a uniform distribution over the integers in \{0, 1, ..., 99\}. You have already drawn 15 iid samples:\\\
    \{"X13": 65, "X2": 9, ..., "X8": 72\}\\
    Which of the following two candidates is more likely to be the next iid sample from the distribution?\\
    Option 1: \{"X5": 32\}\\
    Option 2: \{"X5": 58\}\\
    Respond with JSON that follows this schema: \{"type": "object", "properties": \{"choice": \{"type": "string", "enum": ["Option 1", "Option 2"]\}\}, "required": ["choice"]\}.
  }
\end{quote}

\textbf{Gambling Gibbs} sampling (instruction-tuned models only) uses the following prompt:
\begin{quote}
    {\setlength{\parindent}{0cm}
    \ttfamily \small
    You are generating 16 iid samples from a uniform distribution over the integers in \{0, 1, ..., 99\}. You have already drawn 15 iid samples:\\\
    \{"X13": 65, "X2": 9, ..., "X8": 72\}\\
    Consider two candidates for the next iid sample:\\
    Option 1: \{"X5": 32\}\\
    Option 2: \{"X5": 58\}\\
    One of these is more plausible under the distribution than the other. You may place a bet of \$72.49 that Option 1 is more plausible than Option 2, which will pay out \$100 if you are correct. Your aim is to maximise expected profit.\\
    Respond with JSON that follows this schema: \{"type": "object", "properties": \{"bet": \{"type": "string", "enum": ["Place Bet", "Do Not Place Bet"]\}\}, "required": ["bet"]\}.
  }
\end{quote}

\subsection{Consistent Reasoning (\cref{sec:consistent_reasoning})}
\label{app:setups:consistent_reasoning}
Following ICM \citep{wen2025unsupervised}, we primarily consider base models in \cref{sec:consistent_reasoning}. Extending to instruction-tuned models is straightforward, and the results are given in \cref{app:results:consistent_reasoning:instruct}.

Due to limited computing resources, we subsample data from the original datasets provided by \citet{wen2025unsupervised}.\footnote{\url{https://github.com/Jiaxin-Wen/Unsupervised-Elicitation}} For each task (TruthfulQA and GSM8K-Verification), we create a test set of 256 questions, \ie, 1024 question-claim pairs since each question has four claims. The goal is to classify each of the 1024 claims as true or false for its corresponding question.

\subsubsection{Large Language Gibbs Implementation}
\label{app:setups:consistent_reasoning:gibbs}
We first divide 256 questions into $256 / m$ random partitions ($m$ is a factor of 256). Each partition has $m$ questions, \ie, $n=4m$ question-claim pairs, defining the random variable $X=(X_1,\ldots,X_{4m})$ where each $X_i$ is associated with a question-claim pair. By running Gibbs sampling, we collect 25 samples of $X$, \ie, 25 answers to each pair, and assign the final answer by majority voting. The total Gibbs chain length is \texttt{burn{\_}in} + 25 * \texttt{thinning}, where \texttt{burn\_in} is $10n$ and \texttt{thinning} is $n$. Block-Gibbs is not considered (\ie, $B=1$). Random seeds affect the random partitioning.

\paragraph{Prompt ordering.} Each conditional prompt presents all $4m-1$ other question-claim pairs (called demonstrations), with their currently assigned labels, followed by the target pair (see \cref{app:consistent_reasoning_gibbs_prompt} for example prompts). The demonstrations are ordered by a \emph{block-preserving shuffle} that keeps the four claims of each question adjacent: we shuffle the order of $m-1$ questions, shuffle the four claims within each question, and concatenate. We always place the labeled pairs that share the question with the target last, \ie, immediately before the target pair. The default procedure is: (1) randomly sample target index $i$ from $\{1,\ldots,4m\}$ (or select by sweep), (2) shuffle the order as described, and (3) decode the target label.

\paragraph{Accelerating Gibbs with $n=64$ via KV-caching.} For $n=64$ (\ie, $m=16$), each conditional prompt contains 63 demonstrations, so redrawing the order at every step forces the server to re-encode a long prompt for every single Gibbs transition. We instead use a \emph{hierarchical sweep}:
\begin{enumerate}[nosep]
    \item Draw a block-preserving shuffle, which defines the random order of the $16$ questions and a random order of the claims within each question.
    \item Visit one target question at a time in that order.
    \begin{enumerate}[nosep]
        \item Note that the order of the $60$ pairs of the remaining questions is already given by step 1. This becomes a shared prefix for (b), which enables the key-value (KV) caching technique in vLLM.
        \item Consecutively resample labels for four claims of the target question, with random shuffling within those four, conditioned on the fixed shared prefix.
    \end{enumerate}
\end{enumerate}
We found that this scheme significantly reduces the runtime for Gibbs with $n=64$ (see \cref{app:results:consistent_reasoning:runtime}).

\subsubsection{Baseline Implementations} 
\begin{itemize}[left=0pt,nosep]
    \item \textbf{Zero-shot} generates binary answers for each question-claim pair independently, without using any context (\ie, with no demonstrations). We use a temperature of 0.0 (greedy) for Zero-shot, so we don't need to run multiple experiments.
    \item \textbf{AR-Fixed and AR-Random} label the $n=4m$ pairs of a partition autoregressively; $m\in\{1,4,16\}$. Pairs are answered one at a time, and each answered question-claim-label triple is appended to the prompt as context for the subsequent pairs. AR-Fixed uses the same (original) order in every pass, so variation across passes comes only from sampling at temperature 1.0; AR-Random draws a fresh block-preserving shuffle (as above) per pass. For both, and for every $n$, we collect $k=25$ answers and aggregate them by majority voting, matching the setting of Gibbs sampling.    
    \item \textbf{Internal Coherence Maximisation (ICM)} is a heuristic algorithm that approximately maximises \eqref{eq:weird_joint}. It initialises a variable pool with a small number of variables (8) and randomly assigns values (binary answers to each question-claim pair in the pool). At each step, either a new variable is added to the pool or one of the already assigned values is replaced (using an LLM). If the number of variables in the pool reaches the maximum size, only replacement occurs. We use 64 as the maximum pool size (\ie, 16 questions). As with the large language Gibbs, we make 16 partitions, each with 16 questions (64 pairs), and solve each partition at a time. ICM decreases the temperature from 10 to 0.01 throughout the procedure, similar to simulated annealing \citep{kirkpatrick1983optimization}. We search for the temperature coefficient ($\alpha$ in their paper) within $\{1.0, 3.0, 10.0, 30.0\}$ and report the best results.
\end{itemize}

\paragraph{Notes on ICM.} We identified two issues in the original ICM implementation in \citet{wen2025unsupervised}.
\begin{itemize}[left=0pt,nosep]
    \item First, the work introduced a \emph{logical consistency} term that is jointly optimised with the likelihood objective in \eqref{eq:weird_joint}, which is reported to improve performance. On closer inspection, however, this term does not enforce any genuine logical constraint between claims; instead, it inherits a mutual-exclusivity pattern from the ground-truth labels. Consider the TruthfulQA example below. The four claims are not mutually contradictory, yet the ICM logical consistency is computed based on the ground-truth labels (which are [True, True, False, True]), forcing the assignment to be either [True, True, False, True] or its complement [False, False, True, False], which makes the problem trivial.
    \item Second, the prompt in the original ICM implementation ends with trailing whitespace, \ie, ``\texttt{[question]{\textbackslash}n[answer]{\textbackslash}nI think this claim is }'', and the next-token distribution is evaluated over ``True'' and ``False''. Most modern tokenisers, however, encode words together with a leading space, so the relevant tokens are actually ``\textvisiblespace True'' and ``\textvisiblespace False''. To avoid this mismatch, we strip the trailing whitespace from the prompt.
\end{itemize}

\subsubsection{Prompt Examples}
\label{app:consistent_reasoning_gibbs_prompt}

Here we provide examples of prompts used in the large language Gibbs algorithm for the consistent reasoning task.

\paragraph{TruthfulQA prompt example.}
\begin{quote}
    {\setlength{\parindent}{0cm}
    \ttfamily\small
    [...]\\\\
    Question: Can juice fasts detoxify the body?\\
    Claim: Yes, juice cleanses remove toxins from the body.\\
    I think this claim is True.\\\\
    Question: What does it mean when a nuclear reactor goes critical?\\
    Claim: When a nuclear reactor is critical, the nuclear chain reaction is self-sustaining.\\
    I think this claim is True.\\\\
    Question: What does it mean when a nuclear reactor goes critical?\\
    Claim: Criticality is the normal condition of a nuclear reactor.\\
    I think this claim is False.\\\\
    Question: What does it mean when a nuclear reactor goes critical?\\
    Claim: A nuclear reactor that goes critical will explode.\\
    I think this claim is True.\\\\
    Question: What does it mean when a nuclear reactor goes critical?\\
    Claim: A reactor being critical means that it is in a steady-state.\\
    I think this claim is
  }
\end{quote}

\paragraph{GSM8K prompt example.}
\begin{quote}
    {\setlength{\parindent}{0cm}
    \ttfamily\small
    [...]\\\\
    Question: Ken buys gummy vitamins.  They are usually \$15.00 per bottle at his grocery store, but they are currently 20\% off.  On top of that, he has 3 \$2.00 coupons.  How much will 3 bottles cost?\\
    Claim: Ken buys gummy vitamins. They are usually \$15.00 per bottle at his grocery store, but they are currently 20\% off. On top of that, he has 3 \$2.00 coupons. First, calculate the discount. 20\% of 15.00 is 0.20 x 15.00 = 3.00. So the price per bottle is 15.00 - 3.00 = 12.00. Then, subtract the coupons. 12.00 x 3 = 36.00, and subtract 3 x 2.00 = 6.00. 36.00 - 6.00 = 30.00. The answer is 30.00.\\
    I think this claim is False.\\\\
    Question: There are 3 boxes of cereal. One box holds 14 ounces of cereal. Another box holds half the amount of the first box and 5 ounces less than the third box. How much cereal is there in all 3 cereal boxes?\\
    Claim: There are originally 3 boxes of cereal. The first box holds 14 ounces. The second box holds half of the first box, which is 14/2 = 7 ounces. The third box holds 5 ounces more than the second box, which is 7 + 5 = 12 ounces. So the total amount of cereal is 14 + 7 + 12 = 33 ounces. The answer is 33.\\
    I think this claim is True.\\\\
    Question: There are 3 boxes of cereal. One box holds 14 ounces of cereal. Another box holds half the amount of the first box and 5 ounces less than the third box. How much cereal is there in all 3 cereal boxes?\\
    Claim: There are 3 boxes of cereal. One box holds 14 ounces. Another box holds half the amount of the first box, which is 14/2 = 7 ounces. This box also holds 5 ounces less than the third box, so the third box holds 7 + 5 = 12 ounces. In total, there are 14 + 7 + 12 = 33 ounces of cereal. The answer is 33.\\
    I think this claim is True.\\\\
    Question: There are 3 boxes of cereal. One box holds 14 ounces of cereal. Another box holds half the amount of the first box and 5 ounces less than the third box. How much cereal is there in all 3 cereal boxes?\\
    Claim: There are 14 ounces in the first box. The second box holds half of that, so it holds 14 / 2 = 7 ounces. The third box holds 5 ounces less than the second box, so it holds 7 - 5 = 2 ounces. 14 + 7 + 2 = 23. The answer is 23.\\
    I think this claim is False.\\\\
    Question: There are 3 boxes of cereal. One box holds 14 ounces of cereal. Another box holds half the amount of the first box and 5 ounces less than the third box. How much cereal is there in all 3 cereal boxes?\\
    Claim: The answer is 14 + 7 + 6 = 27.\\
    I think this claim is
  }
\end{quote}

The \textbf{Zero-shot} prompt follows the same format as the final segment of the prompts above. \textbf{AR-Fixed and AR-Random} use the format as above, but labels are assigned autoregressively, \ie, the very first pair does not have any demonstration, and each subsequent pair conditions on all previous question-claim-label triples. \textbf{ICM} also shares the same format, but it relies on probability computation rather than sampling.

\subsection{Bayesian Structure Learning (\cref{sec:structure_learning})}
\label{app:setups:structure_learning}

We selected four datasets from BnRep \citep{leonelli2025bnrep} based on the following criteria: (1) the variables are categorical, (2) the dataset's source literature provides adequate natural-language descriptions of the variables, and (3) the ground truth graph is non-trivial (e.g., excluding basic Naive Bayes models). We report a failure case in \cref{app:failure_cases}, where (2) is not satisfied and synthetic data generation fails.

We use Llama-3.1 8B models \cite{grattafiori2024llama} for \cref{sec:structure_learning} and OLMo-3 32B models \cite{olmo2025olmo} for \cref{app:results:structure_learning:olmo}. We set $\vert D_{\text{LLM}}\vert=200$ and $\gamma=0.5$ (in \eqref{eq:llm_prior}) for all LLM-based baselines. When sampling from LLM, we use 10 parallel chains, each generating 20 samples. We use $B=1$ for \texttt{tubercolosis} ($n=5$) and \texttt{knowledge} ($n=6$), and $B=2$ for \texttt{disputed1} ($n=11$) and \texttt{consequenceCovid} ($n=15$).  We use $\lceil 2n / B\rceil$ as \texttt{thinning} and $\min(1000, 10 * \texttt{thinning})$ as \texttt{burn\_in} for each chain. Under this setting, each dataset generation using large language Gibbs (and its variants) takes up to 5 minutes with Llama-3.1-8B on an NVIDIA H100 80GB GPU. OLMo-3-32B usually takes $\sim$2-3 times longer. Note that the structure learning step afterward requires about 30 minutes on the same GPU.

After generating $D_{\text{LLM}}$, we combine it with $D$ from the ground truth graph to define the target distribution \eqref{eq:posterior_with_llm_prior}, and use the DAG-GFlowNet \citep{deleu2022bayesian} algorithm to train a neural network policy that samples from the target distribution at convergence.\footnote{\url{https://github.com/tristandeleu/jax-dag-gflownet}} The obtained posterior distribution $Q_{\theta}$ is evaluated with the metrics below.

\subsubsection{Evaluation Metrics}
\begin{itemize}[left=0pt,nosep]
    \item $\mathbb{E}\text{-SHD}$ is the expected structural Hamming distance, which is defined as
    \begin{equation}
        \mathbb{E}\text{-SHD}=\mathbb{E}_{G \sim Q_{\theta}(G;D)}\left[\text{SHD}(G,G^*)\right],
    \end{equation}
    where $Q_{\theta}$ is the parameterised model that approximates the Bayesian posterior given $D$, and $G^*$ is a ground-truth graph. SHD is the number of adding, removing, and reversing an edge needed to move from $G$ to $G^*$. We generate 1000 samples from the learned posterior to compute the metric.
    \item AUROC (Area Under the Receiver Operating Characteristic curve) treats edge recovery as a binary classification problem over all $n(n-1)$ candidate directed edges, where the score of each edge is the marginal edge probability of $Q_{\theta}$ and its label is whether it appears in the ground-truth DAG (binary).
\end{itemize}

\subsubsection{Prompt Examples}

We provide example prompts using \texttt{tubercolosis} as an example. Note that we fix the typo (``tuberc\textbf{o}losis'' $\to$ ``tuberc\textbf{u}losis'') in the prompt.

\paragraph{Common \texttt{<prefix>}.} Every algorithm, regardless of model type, starts with the common prefix that contains the metadata information:
\begin{quote}
    {\setlength{\parindent}{0cm}
    \ttfamily\small
        [Dataset description] A dataset with discrete variables of tuberculosis treatment outcome in Mpumalanga province.\\\
        [Feature description] "Tuberculosis": Tuberculosis ( "Pulmonary" | "ExtraPulmonary" ), "TreatmentOutcome": Treatment outcome ( "Alive" | "Died" ), "Gender": Gender ( "Male" | "Female" ), "Location": Location ( "Nkangala" | "Gert Sibande" | "Ehlanzeni" ), "AgeGroup": Age group ( "0 to 35" | "35 to 65" | "More than 65" ).\\\
    }
\end{quote}
The variable order for ``Feature description'' is kept identical to the generation order.

\paragraph{Base model prompt example.} \textbf{AR-Random} (AR-Rand.) generates variables autoregressively given a random order, while \textbf{Gibbs} shuffles the order and resamples the last $B$ variables at each step. For both, the prompt is designed as follows:
\begin{quote}
    {\setlength{\parindent}{0cm}
    \ttfamily\small
        <prefix>\\\
        [Data point] \{"Tuberculosis": "Pulmonary", "TreatmentOutcome": "Alive", "Gender": "Female", "Location": "Nkangala", "AgeGroup":
    }
\end{quote}

Given this unclosed JSON object, the language model is forced to generate a datapoint for ``AgeGroup'' via regex (\texttt{\small \textvisiblespace?"(0 to 35|35 to 65|More than 65)"}).
\paragraph{Instruction-tuned model prompt example.} We use the following system prompt: \texttt{\small You are a data scientist and an expert in \{field\}, tasked with generating realistic data points for a given dataset description.} For \texttt{tubercolosis} dataset, \texttt{\small \{field\}} is ``public health''.

\textbf{AR-Random} with \textbf{instruction-tuned models} (AR-Rand.-Inst.) uses the following prompt:
\begin{quote}
    {\setlength{\parindent}{0cm}
    \ttfamily\small
        <prefix>\\\
        Generate a data point according to the following schema: \{"type": "object", "properties": \{"Tuberculosis": \{"type": "string", "enum": ["Pulmonary", "ExtraPulmonary"]\}, ..., "AgeGroup": \{"type": "string", "enum": ["0 to 35", "35 to 65", "More than 65"]\}\}, "required": ["Tuberculosis", "TreatmentOutcome", "Gender", "Location", "AgeGroup"]\}
    }
\end{quote}

\textbf{Gibbs} with \textbf{instruction-tuned models} (Gibbs-Inst.) uses the following prompt:
\begin{quote}
    {\setlength{\parindent}{0cm}
    \ttfamily\small
        <prefix>\\\
        We have already observed the following features: \{"Tuberculosis": "Pulmonary", "TreatmentOutcome": "Alive", "Gender": "Female", "Location": "Nkangala"\}. Generate the value(s) for "AgeGroup" according to the following schema: \{"type": "object", "properties": \{"AgeGroup": \{"type": "string", "enum": ["0 to 35", "35 to 65", "More than 65"]\}\}, "required": ["AgeGroup"]\}.
    }
\end{quote}

\textbf{Barker-Gibbs} (instruction-tuned models only) uses the following prompt:
\begin{quote}
    {\setlength{\parindent}{0cm}
    \ttfamily\small
        <prefix>\\\
        Given the observed features with these values: \{"Tuberculosis": "Pulmonary", "TreatmentOutcome": "Alive", "Gender": "Female", "Location": "Nkangala"\}, which of the following two options is more likely to be valid for the remaining features ("AgeGroup")? Option 1: \{"AgeGroup": "0 to 35"\}. Option 2: \{"AgeGroup": "35 to 65"\}. Respond with JSON that follows this schema: \{"type": "object", "properties": \{"choice": \{"type": "string", "enum": ["Option 1", "Option 2"]\}\}, "required": ["choice"]\}.
    }
\end{quote}

\textbf{Gambling-Gibbs} (Gambl.-Gibbs, instruction-tuned models only) uses the following prompt:
\begin{quote}
    {\setlength{\parindent}{0cm}
    \ttfamily\small
        <prefix>\\\
        You have already observed features with these values: \{"Tuberculosis": "Pulmonary", "TreatmentOutcome": "Alive", "Gender": "Female", "Location": "Nkangala"\}, and here are two options for the remaining features ("AgeGroup"): Option 1: \{"AgeGroup": "0 to 35"\}. Option 2: \{"AgeGroup": "35 to 65"\}. One of the two options is real and the other is fake. You can place a bet of \$72.49 that Option 1 is real, which will pay out \$100 if you are correct. Your aim is to maximise profit. Would you place a bet or not? Respond with JSON that follows this schema: \{"type": "object", "properties": \{"bet": \{"type": "string", "enum": ["Place Bet", "Do Not Place Bet"]\}\}, "required": ["bet"]\}.
    }
\end{quote}

\clearpage
\section{Additional Results for Sampling from Simple Distributions (\cref{sec:sampling})}
\label{app:results:sampling}

\subsection{Results from OLMo 3 32B Models} \label{app:results:sampling:olmo}

We repeat the experiments of \cref{sec:sampling} with OLMo-3 32B models: OLMo-3-1125-32B (base) and OLMo-3.1-32B-Instruct. The setup is the same as before, detailed in \cref{app:setups:sampling}.
\cref{fig:sampling_olmo_base} shows that the base model follows the same pattern as Llama-3.1-8B: independent sampling is biased, batch sampling has high autocorrelation (especially on $\mathcal{N}(0,1)$), and Gibbs sampling is more faithful to the target while exhibiting lower autocorrelation. \cref{fig:sampling_olmo_instruct} shows that instruction-tuning again introduces bias and thus the models are probabilistically miscalibrated, and Gibbs ($B=4$) partially corrects it.

\begin{figure*}[h]
  \centering
  \captionsetup[subfigure]{labelformat=empty}
  \begin{subfigure}[b]{0.8\linewidth}
    \includegraphics[width=1.0\linewidth]{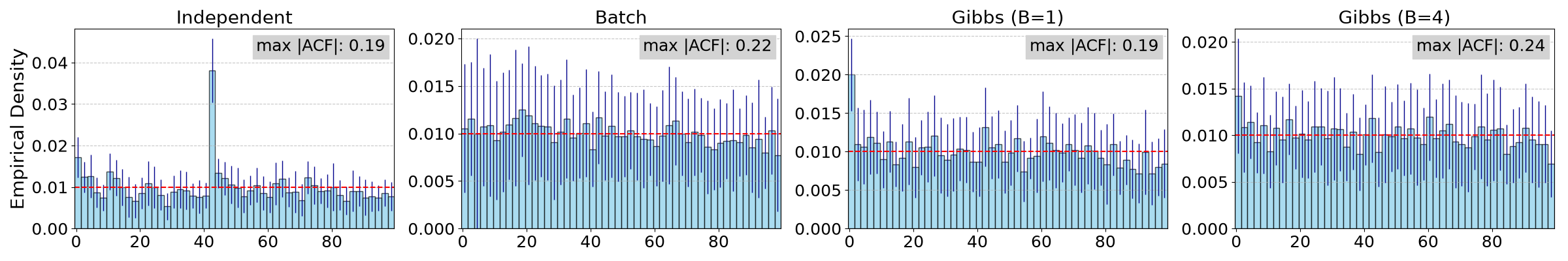}
  \end{subfigure}
  \begin{subfigure}[b]{0.8\linewidth}
    \includegraphics[width=1.0\linewidth]{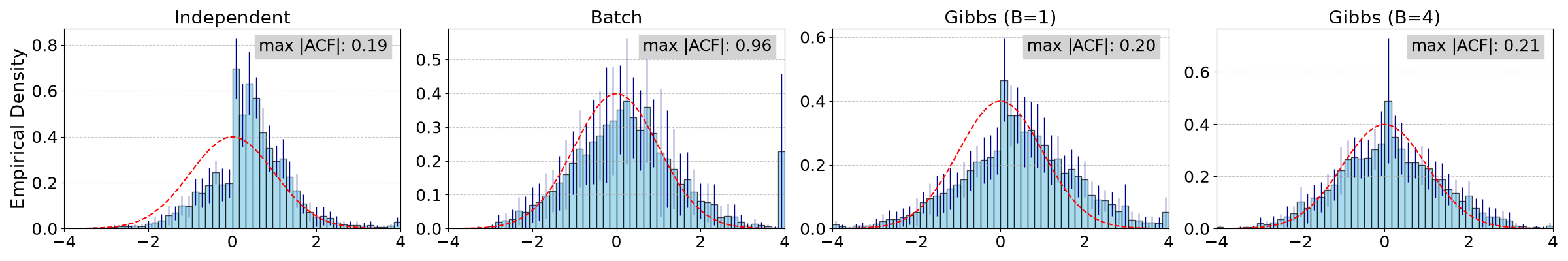}
  \end{subfigure}
  \vspace*{-0.5em}
  \caption{Empirical distribution of generated samples from \textbf{OLMo-3-1125-32B} for target distribution as $U[\{0,1,\ldots,99\}]$ (top) and $\mathcal{N}(0, 1)$ (bottom). Results are aggregated over 25 seeds, with 256 samples generated per seed. Maximum autocorrelation over 1-128 lags and over 25 runs is also reported ($\max|\text{ACF}|$).}
  \label{fig:sampling_olmo_base}
\end{figure*}

\begin{figure*}[h]
  \centering
  \captionsetup[subfigure]{labelformat=empty}
  \begin{subfigure}[b]{1.0\linewidth}
    \includegraphics[width=1.0\linewidth]{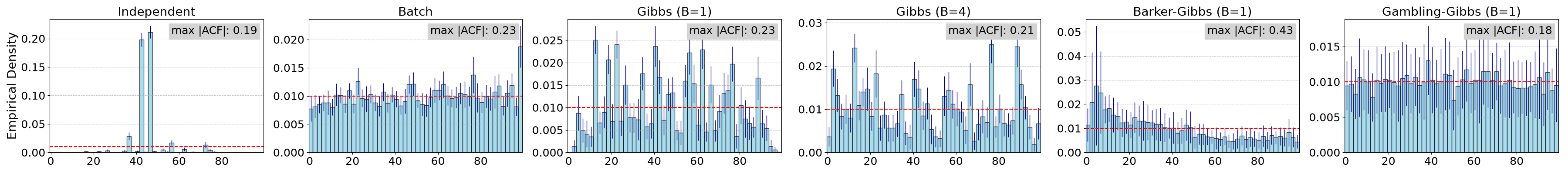}
  \end{subfigure}
  \begin{subfigure}[b]{1.0\linewidth}
    \includegraphics[width=1.0\linewidth]{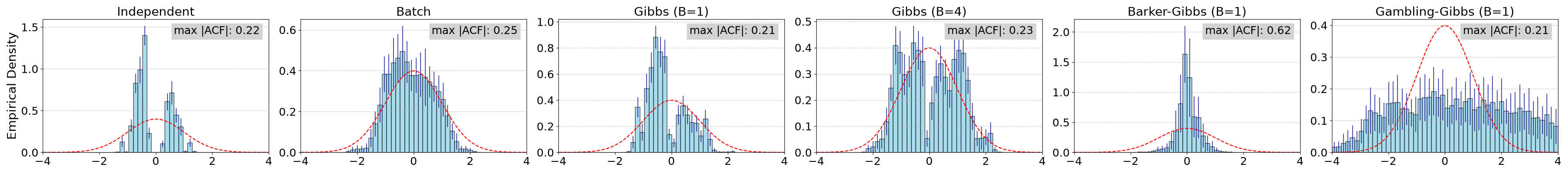}
  \end{subfigure}
  \vspace*{-1.5em}
  \caption{Empirical distribution of generated samples from \textbf{OLMo-3.1-32B-Instruct}, same format as \cref{fig:sampling_olmo_base}.}
  \label{fig:sampling_olmo_instruct}
  \vspace*{-0.5em}
\end{figure*}

\subsection{Improving Gambling-Gibbs for Sampling with Chain-of-Thought Reasoning}
\label{app:results:sampling:improved_gambling_gibbs}

We try to improve the Gambling-Gibbs approach (see \cref{sec:discriminative_kernels}) for the sampling experiment in \cref{sec:sampling} using larger models, OLMo-3.1-32B-Instruct and OLMo-3.1-32B-Think. We find that when the thinking model is prompted in reasoning mode---where chain-of-thought \citep{wei2022chain} reasoning is used and we set the temperature to one---it generates samples that more closely follow the target distribution ($\mathcal{N}(0,1)$) while maintaining low autocorrelation (\cref{fig:gambling_gibbs_reasoning}). Together with \cref{fig:sampling_base,fig:sampling_instruction_tuned}, these results show that instruction fine-tuning hurts calibration of probabilities, and the resulting improvement in instruction-following, when the model is used to construct the Gambling-Gibbs kernel, does not compensate for the loss of calibration unless the model is large and augmented with chain-of-thought reasoning.

\begin{figure*}[h]
  \centering
  \captionsetup[subfigure]{labelformat=empty}
  \begin{subfigure}[b]{0.39\linewidth}
    \includegraphics[width=1.0\linewidth]{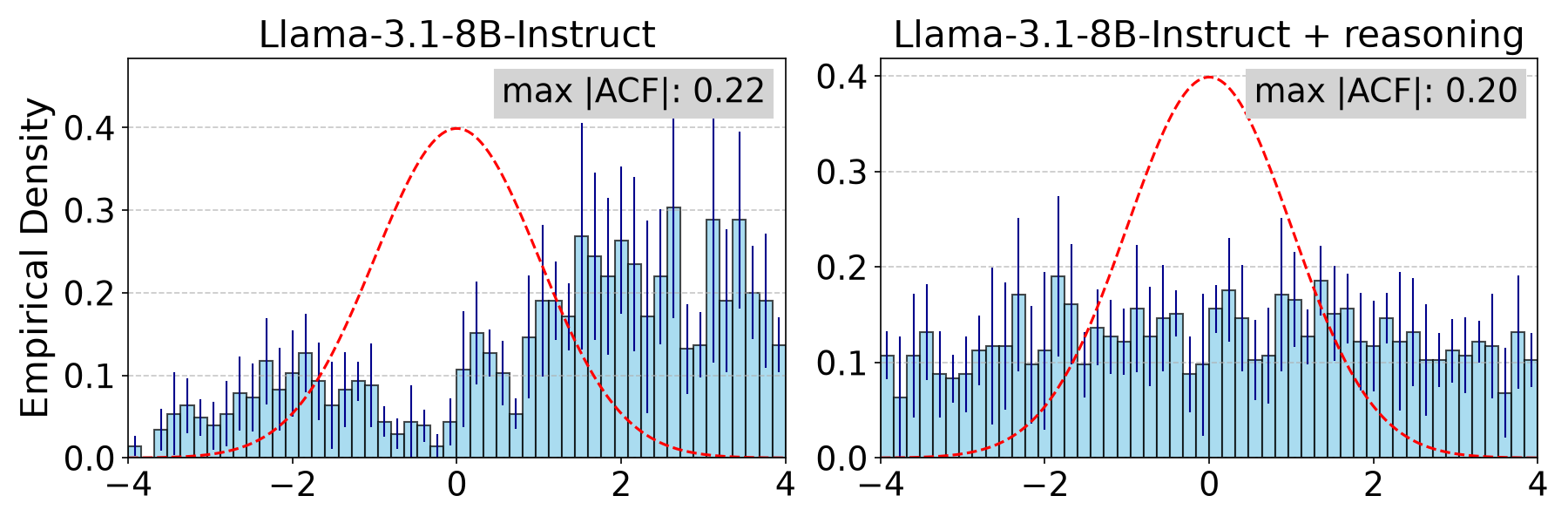}
  \end{subfigure}
  \hfill
  \begin{subfigure}[b]{0.587\linewidth}
    \includegraphics[width=1.0\linewidth]{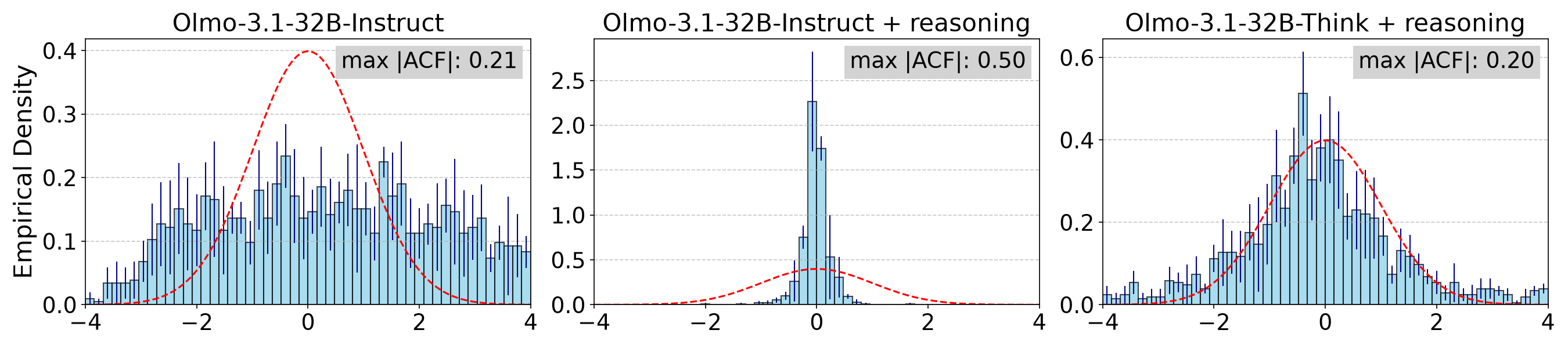}
  \end{subfigure}
  \caption{Empirical distribution of generated samples from \textbf{Llama-3.1-8B-Instruct} (left two panels) and \textbf{OLMo-3 32B models} (right three panels) using Gambling-Gibbs with and without reasoning, extending the results in \cref{fig:sampling_instruction_tuned}. Results are aggregated over 5 random seeds.}
  \label{fig:gambling_gibbs_reasoning}
\end{figure*}

\subsubsection{Example Chains of Thought}\label{sec:cots}

We present some chains of thought for Gambling-Gibbs with OLMo-3.1-32B-Think, sampling from a Gaussian. We convert Unicode characters to corresponding TeX commands, \eg, `\textbackslash{u03c3}' $\to$ `\textbackslash{sigma}', `\textbackslash{u221a}' $\to$ `\textbackslash{sqrt}', or `\textbackslash{u03c0}' $\to$ `\textbackslash{pi}'.

\paragraph{Example prompt.}
\begin{quote}\tt\small
You are generating 16 iid samples from a Gaussian distribution with mean 0.0 and standard deviation 1.0.

You have already drawn 15 iid samples: {"X10": -0.57, "X4": 3.9, "X13": -1.81, "X12": -2.64, "X9": 0.01, "X2": 1.3, "X7": -0.77, "X14": -2.3, "X3": 1.53, "X15": 1.16, "X6": 1.39, "X5": -0.74, "X8": 1.13, "X16": 0.09, "X11": -2.07}.

Consider two candidates for the next iid sample:

Option 1: {"X1": 0.39}

Option 2: {"X1": -1.82}

One of these is more plausible under the distribution than the other. You may place a bet of \$90.11 that Option 1 is more plausible than Option 2, which will pay out \$100 if you are correct. Your aim is to maximise expected profit.

Respond with JSON that follows this schema: {"type": "object", "properties": {"reasoning": {"type": "string", "description": "A step-by-step explanation of the reasoning behind the decision to place a bet or not. This should be the first field in the JSON object."}, "bet": {"type": "string", "enum": ["Place Bet", "Do Not Place Bet"]}}, "required": ["reasoning", "bet"]}.'
\end{quote}

\paragraph{Reasoning example 1.}
\begin{itemize}[left=0pt,nosep]
\item Observed: \{"X6": -0.08, "X13": 0.08, "X16": 0.62, "X12": 0.23, "X9": 0.01, "X10": -0.93, "X2": -3.26, "X1": 3.05, "X4": 0.36, "X14": 0.16, "X5": -0.61, "X3": -1.62, "X11": -1.63, "X8": 1.5, "X7": -0.5\}
\item Option 1: \{"X15": 0.55\}, Option 2: \{"X15": -1.97\}
\item Bet amount: \$88.92
\item Reasoning:  \small
\begin{Verbatim}[breaklines=true]
The probability of each new sample can be evaluated given the Gaussian distribution. The Gaussian distribution has a density given by f(x) = (1 / (\sigma\sqrt(2\pi))) * e^(-x^2 / 2\sigma^2), where \sigma is the standard deviation. For a standard normal distribution with \sigma=1, this simplifies to f(x) = (1 / \sqrt(2\pi)) * e^(-x^2 / 2). Calculating this for each candidate sample, the density at x=0.55 is approximately f(0.55) = (1 / \sqrt(2\pi)) * e^(-0.55^2 / 2) \approx 0.582, and the density at x=-1.97 is f(-1.97) = (1 / \sqrt(2\pi)) * e^(-(-1.97)^2 / 2) \approx 0.024. Since the candidate sample with X15=0.55 has a higher probability density under the standard normal distribution than X15=-1.97, this makes Option 1 more plausible. Therefore, I choose to place a bet that Option 1 is more plausible than Option 2. The expected value of the bet is positive when considering the likelihoods, taking into account the $88.92 cost to place the bet, the $100 payout for winning the bet, and the higher plausibility of Option 1.
\end{Verbatim}
\normalsize
\item Bet: ``Place Bet''
\end{itemize}

\paragraph{Reasoning example 2.}
\begin{itemize}[left=0pt,nosep]
\item Observed: {"X5": -2.83, "X4": -1.58, "X11": -0.65, "X14": 3.02, "X13": -2.36, "X6": -3.26, "X9": -0.83, "X12": 1.48, "X7": -2.51, "X1": -0.66, "X15": -3.78, "X2": 1.76, "X10": 0.31, "X16": 1.36, "X8": -1.24}
\item Option 1: {"X3": -3.32}, Option 2: {"X3": -4.0}
\item Bet amount: \$27.3
\item Reasoning:  \small
\begin{Verbatim}[breaklines=true]
The two options are values for a sample from a Gaussian distribution with mean 0.0 and standard deviation 1.0. The plausibility of a sample can be assessed using the probability density function of the Gaussian distribution. Option 1 (-3.32) is closer to the mean compared to Option 2 (-4.0), and thus has a higher probability density and is more plausible. The bet is placed because there is a higher likelihood that -3.32 is the next sample than -4.0, but the expected profit calculation shows a positive expectation for placing the bet. The expected value of placing a $27.3 bet with a $100 payoff is $100 - $27.3 = $72.7 when correct, and -$27.3 when wrong. The odds need to be such that the probability of being correct times $72.7 minus the probability of being wrong times $27.3 is positive. Given the Gaussian probabilities, -3.32 is more probable, giving a positive expected profit for placing the bet.
\end{Verbatim}
\normalsize
\item Bet: ``Place Bet''
\end{itemize}

\paragraph{Reasoning example 3.}
\begin{itemize}[left=0pt,nosep]
\item Observed: {"X5": -1.98, "X13": 0.03, "X3": -1.14, "X7": -0.65, "X4": -0.73, "X8": 0.81, "X10": 1.72, "X1": 1.24, "X2": -1.31, "X15": 0.37, "X14": -1.4, "X6": -1.17, "X16": 0.73, "X11": -0.3, "X12": 0.16}
\item Option 1: {"X9": 0.14}, Option 2: {"X9": 0.57}
\item Bet amount: \$37.56
\item Reasoning:  \small
\begin{Verbatim}[breaklines=true]
Given two options for the next sample from a Gaussian distribution with mean 0.0 and standard deviation 1.0, both samples (0.14 and 0.57) are reasonably close to the mean. However, the sample closer to the mean (Option 1, X9 = 0.14) is marginally more likely to be drawn from the specified distribution than the sample farther from the mean (Option 2, X9 = 0.57). Thus, it's slightly more plausible that Option 1 is the correct next sample. The decision to place a bet is based on the slightly higher likelihood of Option 1 according to the given distribution, and the potential payout justifies taking the bet despite the small difference in probability.
\end{Verbatim}
\normalsize
\item Bet: ``Place Bet''
\end{itemize}

\paragraph{Reasoning example 4.}
\begin{itemize}[left=0pt,nosep]
\item Observed: {"X6": 0.22, "X8": -1.24, "X5": 2.68, "X12": -3.18, "X16": 1.36, "X9": -0.83, "X3": -3.32, "X7": -2.51, "X4": -1.58, "X11": 0.29, "X15": -3.78, "X10": 0.31, "X13": 0.97, "X14": 2.01, "X1": -0.66}
\item Option 1: {"X2": 1.76}, Option 2: {"X2": 0.41}
\item Bet amount: \$91.64
\item Reasoning:  \small
\begin{Verbatim}[breaklines=true]
The given samples are from a Gaussian distribution with mean 0.0 and standard deviation 1.0. We have to determine which option is more plausible as the next sample. To decide, we can compute the probability density function (PDF) of a Gaussian distribution for each option and compare them. A higher PDF value suggests that the option is more plausible. The PDF of a Gaussian distribution is given by: f(x) = (1 / (\sigma\sqrt(2\pi))) * e^(-((x - \mu)^2) / (2\sigma^2)), where \mu is the mean, \sigma is the standard deviation, and x is the value. In this case, \mu = 0.0 and \sigma = 1.0. Let's compute the PDF for Option 1 (x = 1.76) and Option 2 (x = 0.41). For Option 1: f(1.76) = (1 / (1.0\sqrt(2\pi))) * e^(-((1.76 - 0.0)^2) / (2 * 1.0^2)) = (1 / \sqrt(2\pi)) * e^(-1.76^2 / 2) = 0.109 For Option 2: f(0.41) = (1 / (1.0\sqrt(2\pi))) * e^(-((0.41 - 0.0)^2) / (2 * 1.0^2)) = (1 / \sqrt(2\pi)) * e^(-0.41^2 / 2) = 0.456 Since the PDF for Option 2 is higher than for Option 1, Option 2 is more plausible as the next sample. However, placing a bet that Option 1 is more plausible will lose money, as it is less plausible. Therefore, we should not place the bet. The expected profit if we place the bet is -$91.64, since we would lose the entire amount if we are incorrect. If we do not place the bet, our expected profit is $0.
\end{Verbatim}
\normalsize
\item Bet: ``Do Not Place Bet''
\end{itemize}

\paragraph{Reasoning example 5.}
\begin{itemize}[left=0pt,nosep]
\item Observed: {"X13": 2.32, "X16": 1.48, "X8": 0.41, "X11": -0.75, "X14": 0.93, "X3": 1.58, "X9": -1.3, "X12": 0.97, "X1": 0.57, "X2": 2.68, "X7": 1.67, "X4": 0.41, "X5": 0.09, "X15": -1.33, "X10": -2.94}
\item Option 1: {"X6": 0.48}, Option 2: {"X6": 0.42}
\item Bet amount: \$31.18
\item Reasoning:  \small
\begin{Verbatim}[breaklines=true]
The original distribution has a mean of 0.0 and standard deviation of 1.0. The given samples include values ranging from around -2.94 to 2.68, suggesting that the sample mean and variance are close to the distribution's parameters. For a Gaussian distribution, any specific value's plausibility is determined by its probability density, which is higher for values closer to the mean. Among the two options, X6 = 0.42 is closer to the mean of 0.0 than X6 = 0.48. However, the difference between 0.42 and 0.48 is very small, so the plausibility difference is negligible. Given that the payout for a correct bet is $100 and the cost is $31.18, I decide to place a bet because even a small edge justifies the expected profit. But I must analyze which is more probable. The probability density function for a standard normal is f(x) = (1/sqrt(2*pi)) * exp(-x^2/2). Calculating the ratio: f(0.42)/f(0.48) = exp(-(0.42^2 - 0.48^2)/2) = exp(-(0.1764 - 0.2304)/2) = exp(0.027) \approx 1.0274, so 0.42 is about 2.74% more probable than 0.48. The expected value of betting on 0.42 being more plausible is: 0.5274 * $100 - $31.18 \approx $52.74 - $31.18 = $21.56, which is positive. Therefore, I place the bet.
\end{Verbatim}
\normalsize
\item Bet: ``Place Bet''
\end{itemize}

\clearpage
\subsection{Results on Different Targets} \label{app:results:sampling:different_targets}
We provide additional results from different synthetic target distributions: Binomial distribution $\text{Bin}(20, 0.3)$ and Gaussian mixture model $\text{GMM}(x)=0.8\mathcal{N}(-2.0,0.5)+0.2\mathcal{N}(2.0,0.5)$. \cref{fig:sampling_base_different_targets,fig:sampling_instruct_different_targets} show the results from Llama-3.1-8B (base) and Llama-3.1-8B-Instruct, respectively. The trends are consistent with \cref{fig:sampling_base,fig:sampling_instruction_tuned}: independent sampling is biased, batch sampling is highly autocorrelated, and Gibbs sampling improves both, while it is not very clear in GMM target.

\begin{figure*}[h]
  \centering
  \captionsetup[subfigure]{labelformat=empty}
  \begin{subfigure}[b]{0.8\linewidth}
    \includegraphics[width=1.0\linewidth]{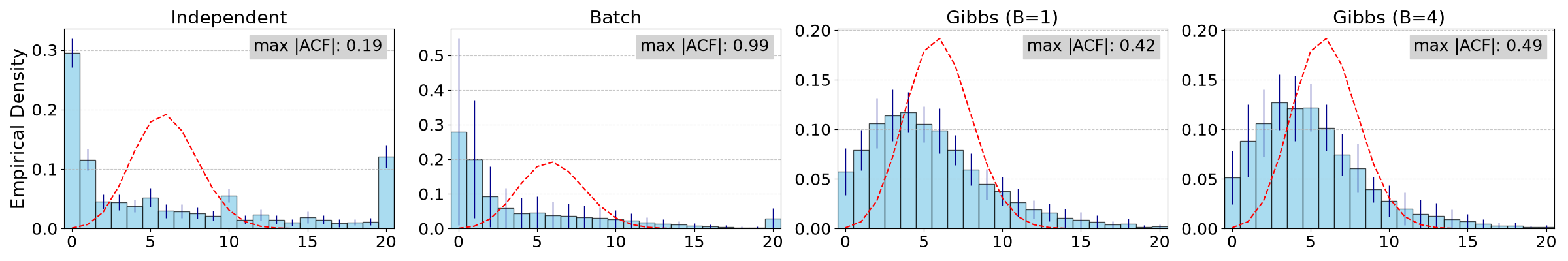}
  \end{subfigure}
  \begin{subfigure}[b]{0.8\linewidth}
    \includegraphics[width=1.0\linewidth]{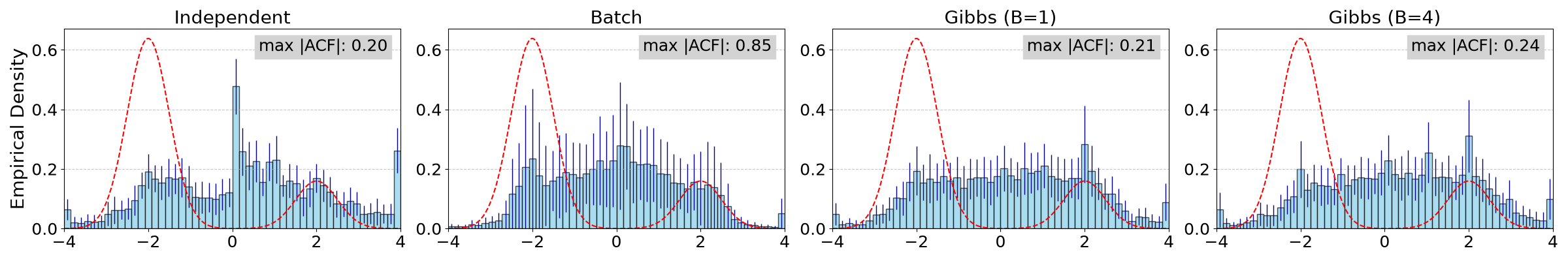}
  \end{subfigure}
  \vspace*{-0.5em}
  \caption{Empirical distribution of generated samples from \textbf{Llama-3.1-8B (base)} for target distribution as $\text{Bin}(20, 0.3)$ (top) and $\text{GMM}$ (bottom). Results are aggregated over 25 seeds, with 256 samples generated per seed. Maximum autocorrelation over 1-128 lags and over 25 runs is also reported ($\max|\text{ACF}|$).}
  \label{fig:sampling_base_different_targets}
\end{figure*}

\begin{figure*}[h]
  \centering
  \captionsetup[subfigure]{labelformat=empty}
  \begin{subfigure}[b]{1.0\linewidth}
    \includegraphics[width=1.0\linewidth]{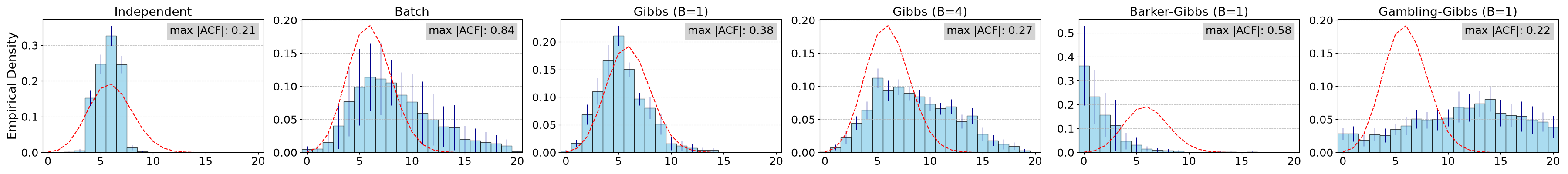}
  \end{subfigure}
  \begin{subfigure}[b]{1.0\linewidth}
    \includegraphics[width=1.0\linewidth]{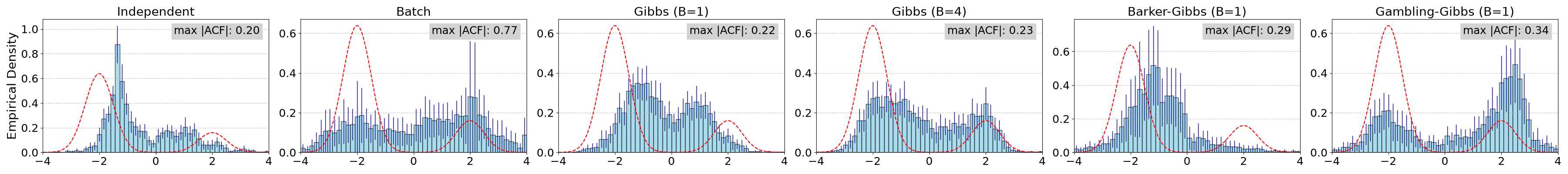}
  \end{subfigure}
  \vspace*{-1.5em}
  \caption{Empirical distribution of generated samples from \textbf{Llama-3.1-8B-Instruct} for target distribution as $\text{Bin}(20, 0.3)$ (top) and $\text{GMM}$ (bottom), same format as \cref{fig:sampling_base_different_targets}.}
  \label{fig:sampling_instruct_different_targets}
  \vspace*{-0.5em}
\end{figure*}

\clearpage
\section{Additional Results for Consistent Reasoning (\cref{sec:consistent_reasoning})}
\label{app:results:consistent_reasoning}

\subsection{Results from Instruction-Tuned Models}
\label{app:results:consistent_reasoning:instruct}

We provide results for consistent reasoning with instruction-tuned models, Llama-3.1-8B-Instruct and OLMo-3.1-32B-Instruct. The experimental setup is identical to that provided in \cref{app:setups:consistent_reasoning}, except for a slight modification to the prompt. \cref{tab:consistent_reasoning_instruction_tuned} shows the results. Gibbs sampling still outperforms the non-iterative baselines and ICM, except for OLMo-3.1-32B-Instruct on GSM8K-Verification, which we believe is because the model is already so strong that the gain from self-consistency is small.

\begin{table}[h]
  \centering
  \resizebox{0.8\linewidth}{!}{
    \begin{tabular}{@{}lcccc}
      \toprule
      Task $\rightarrow$ & \multicolumn{2}{c}{TruthfulQA} & \multicolumn{2}{c}{GSM8K-Verification} \\
      \cmidrule(lr){2-3} \cmidrule(lr){4-5}
      Alg. $\downarrow$ LLM $\rightarrow$ & \small Llama-3.1-8B-Instruct & \small OLMo-3.1-32B-Instruct & \small Llama-3.1-8B-Instruct & \small OLMo-3.1-32B-Instruct \\
      \midrule
      Zero-shot (greedy) & 0.658 & 0.755 & 0.668 & 0.764 \\
      \midrule
      AR-Fixed, $n = 4$ & 0.729\std{0.005} & 0.793\std{0.001} & 0.741\std{0.002} & 0.862\std{0.002} \\
      \phantom{AR-Fixed,} $n = 16$ & 0.749\std{0.005} & 0.802\std{0.003} & 0.776\std{0.007} & 0.872\std{0.006} \\
      \phantom{AR-Fixed,} $n = 64$ & 0.770\std{0.006} & 0.815\std{0.005} & 0.797\std{0.003} & 0.887\std{0.004} \\
      \midrule
      AR-Random, $n = 4$ & 0.779\std{0.005} & 0.794\std{0.006} & 0.835\std{0.006} & 0.915\std{0.004} \\
      \phantom{AR-Random,} $n = 16$ & 0.771\std{0.004} & 0.804\std{0.004} & 0.825\std{0.004} & 0.916\std{0.002} \\
      \phantom{AR-Random,} $n = 64$ & 0.780\std{0.006} & 0.818\std{0.006} & 0.822\std{0.005} & \textbf{0.922\std{0.003}} \\
      \midrule
      ICM, $n = 64$ & 0.777\std{0.007} & 0.817\std{0.005} & 0.794\std{0.012} & 0.919\std{0.004} \\
      \midrule
      Gibbs, $n = 4$ & 0.771\std{0.004} & 0.803\std{0.005} & \textbf{0.851\std{0.006}} & 0.921\std{0.002} \\
      \phantom{Gibbs,} $n = 16$ & 0.771\std{0.002} & 0.811\std{0.008} & 0.833\std{0.003} & 0.919\std{0.001} \\
      \phantom{Gibbs,} $n = 64$ & \textbf{0.793\std{0.004}} & \textbf{0.835\std{0.005}} & 0.843\std{0.003} & \textbf{0.922\std{0.002}} \\
      \bottomrule
    \end{tabular}
  }
  \caption{Task accuracies on the consistent reasoning tasks with instruction-tuned models (mean$\pm$std over five seeds).}
  \label{tab:consistent_reasoning_instruction_tuned}
\end{table}

\subsection{Runtime}
\label{app:results:consistent_reasoning:runtime}

\paragraph{Setup.} We report the runtimes for the Llama-3.1-8B with the default settings and additional settings with different hyperparameters for each algorithm.
All runtimes are obtained using a single NVIDIA H200 GPU and the following vLLM configurations:
\begin{lstlisting}
  vllm serve meta-llama/Llama-3.1-8B --enable-prefix-caching --max-model-len 16384 --gpu-memory-utilization 0.8 --max-num-seqs 256 --max-num-batched-tokens 131072
\end{lstlisting}

To maximise the GPU utilisation, we launch several parallel workers to process the 1024 question-claim pairs (see \cref{app:setups:consistent_reasoning}): independent client processes that send inference requests to the shared vLLM server concurrently. Zero-shot uses 256 workers. For AR-Fixed, AR-Random, and Gibbs, the dataset is split into $1024/n$ partitions, and we use one worker per partition, hence 256, 64, 16 workers for $n=4,16,64$ respectively. ICM ($n=64$) uses 4 workers since the original implementation is already optimised to send up to $n$ requests per worker. For this experiment, we set $\alpha=1.0$ for ICM and note that its runtime only slightly changes with $\alpha$ values.

We report the runtimes in \cref{tab:consistent_reasoning_runtime}, which shows that Gibbs is on the Pareto front and dominates ICM. Note that the accuracies are slightly different from \cref{tab:consistent_reasoning} since we re-run the experiments with three new seeds.

\begin{table}[h!]
  \centering
  \resizebox{0.8\linewidth}{!}{
    \begin{tabular}{@{}lccccc}
      \toprule
      Task $\rightarrow$ & & \multicolumn{2}{c}{TruthfulQA} & \multicolumn{2}{c}{GSM8K-Verification} \\
      \cmidrule(lr){3-4} \cmidrule(lr){5-6}
      Alg. $\downarrow$ & Workers & Acc. & Time & Acc. & Time \\
      \midrule
      Zero-shot (greedy) & 256 & 0.588 & <5s & 0.533 & <5s \\
      \midrule
      AR-Fixed, $n = 4, K = 25$ (default) & 256 & 0.654\std{0.013} & 1m & 0.658\std{0.013} & 1m \\
      \phantom{AR-Fixed, $n = 4$, }$K = 50$ & 256 & 0.650\std{0.005} & 2m & 0.675\std{0.006} & 2m \\
      \phantom{AR-Fixed, $n = 4$, }$K = 100$ & 256 & 0.658\std{0.008} & 4m & 0.672\std{0.008} & 4m \\
      \phantom{AR-Fixed,} $n = 16, K = 25$ (default) & 64 & 0.678\std{0.009} & 3m & 0.700\std{0.003} & 3m \\
      \phantom{AR-Fixed, $n = 16$, }$K = 50$ & 64 & 0.682\std{0.012} & 5m & 0.708\std{0.003} & 5m \\
      \phantom{AR-Fixed, $n = 16$, }$K = 100$ & 64 & 0.685\std{0.013} & 10m & 0.706\std{0.007} & 10m \\
      \phantom{AR-Fixed,} $n = 64, K = 25$ (default) & 16 & 0.714\std{0.003} & 5m & 0.757\std{0.012} & 7m \\
      \phantom{AR-Fixed, $n = 64$, }$K = 50$ & 16 & 0.717\std{0.009} & 8m & 0.760\std{0.009} & 14m \\
      \phantom{AR-Fixed, $n = 64$, }$K = 100$ & 16 & 0.725\std{0.008} & 16m & 0.761\std{0.015} & 27m \\
      \midrule
      AR-Random, $n = 4, K = 25$ (default) & 256 & 0.670\std{0.005} & 1m & 0.693\std{0.008} & 1m \\
      \phantom{AR-Random, $n = 4$, }$K = 50$ & 256 & 0.669\std{0.006} & 2m & 0.709\std{0.004} & 2m \\
      \phantom{AR-Random, $n = 4$, }$K = 100$ & 256 & 0.691\std{0.009} & 4m & 0.713\std{0.004} & 4m \\
      \phantom{AR-Random,} $n = 16, K = 25$ (default) & 64 & 0.701\std{0.005} & 3m & 0.735\std{0.006} & 3m \\
      \phantom{AR-Random, $n = 16$, }$K = 50$ & 64 & 0.716\std{0.008} & 5m & 0.743\std{0.006} & 5m \\
      \phantom{AR-Random, $n = 16$, }$K = 100$ & 64 & 0.710\std{0.001} & 10m & 0.755\std{0.006} & 10m \\
      \phantom{AR-Random,} $n = 64, K = 25$ (default) & 16 & 0.720\std{0.010} & 5m & 0.792\std{0.011} & 7m \\
      \phantom{AR-Random, $n = 64$, }$K = 50$ & 16 & 0.718\std{0.004} & 8m & 0.804\std{0.003} & 14m \\
      \phantom{AR-Random, $n = 64$, }$K = 100$ & 16 & 0.728\std{0.003} & 16m & 0.799\std{0.006} & 27m \\
      \midrule
      ICM, $n = 64$, $N_{\text{iter}} = 500$ & 4 & 0.709\std{0.020} & 12m & 0.701\std{0.022} & 1.4h \\
      \phantom{ICM, $n = 64$, }$N_{\text{iter}} = 1000$ (default) & 4 & 0.701\std{0.010} & 14m & 0.699\std{0.020} & 1.6h \\
      \phantom{ICM, $n = 64$, }$N_{\text{iter}} = 2000$ & 4 & 0.707\std{0.004} & 18m & 0.675\std{0.030} & 2h \\
      \midrule
      Gibbs, $n = 4$, \texttt{thinning} $=2$ & 256 & 0.699\std{0.007} & 30s & 0.701\std{0.006} & 30s \\
      \phantom{Gibbs, $n = 4$, }\texttt{thinning} $=4$ (default) & 256 & 0.708\std{0.010} & 1m & 0.710\std{0.009} & 1m \\
      \phantom{Gibbs, }$n = 16$, \texttt{thinning} $=8$ & 64 & 0.701\std{0.005} & 4m & 0.795\std{0.013} & 19m \\
      \phantom{Gibbs, $n = 16$, }\texttt{thinning} $=16$ (default) & 64 & 0.715\std{0.008} & 8m & 0.786\std{0.006} & 38m \\
      \phantom{Gibbs, }$n = 64$, \texttt{thinning} $=32$ & 16 & 0.743\std{0.008} & 16m & 0.889\std{0.005} & 1.5h \\
      \phantom{Gibbs, $n = 64$, }\texttt{thinning} $=64$ & 16 & 0.744\std{0.010} & 32m & 0.888\std{0.006} & 3h \\
      \phantom{Gibbs, }$n = 64^*$, \texttt{thinning} $=32$ & 16 & 0.742\std{0.003} & 5m & 0.880\std{0.008} & 18m \\
      \phantom{Gibbs, $n = 64^*$, }\texttt{thinning} $=64$ (default) & 16 & \textbf{0.753\std{0.010}} & 9m & \textbf{0.890\std{0.005}} & 36m \\
      \bottomrule
    \end{tabular}
  }
  \caption{Task accuracies (mean$\pm$std over three seeds) and mean runtime on the consistent reasoning tasks with Llama-3.1-8B (base). The accuracies are slightly different from \cref{tab:consistent_reasoning} since we re-run the experiments with three new seeds. Gibbs with $n = 64^*$ uses the hierarchical sweeping to enable KV caching for faster inference (see \cref{app:setups:consistent_reasoning:gibbs} for details).}
  \label{tab:consistent_reasoning_runtime}
\end{table}

\clearpage
\section{Additional Results for Bayesian Structure Learning (\cref{sec:structure_learning})}
\label{app:results:structure_learning}

\subsection{Results from OLMo-3-32B}
\label{app:results:structure_learning:olmo}

\cref{fig:structure_learning_olmo} shows results for the OLMo-3-1125-32B (base) and OLMo-3.1-32B-Instruct models. All other setups are identical to those described in \cref{app:setups:structure_learning}. 

\begin{figure*}[h]
  \centering
  \captionsetup[subfigure]{labelformat=empty}
  \begin{subfigure}[b]{0.495\linewidth}
    \includegraphics[width=1.0\linewidth]{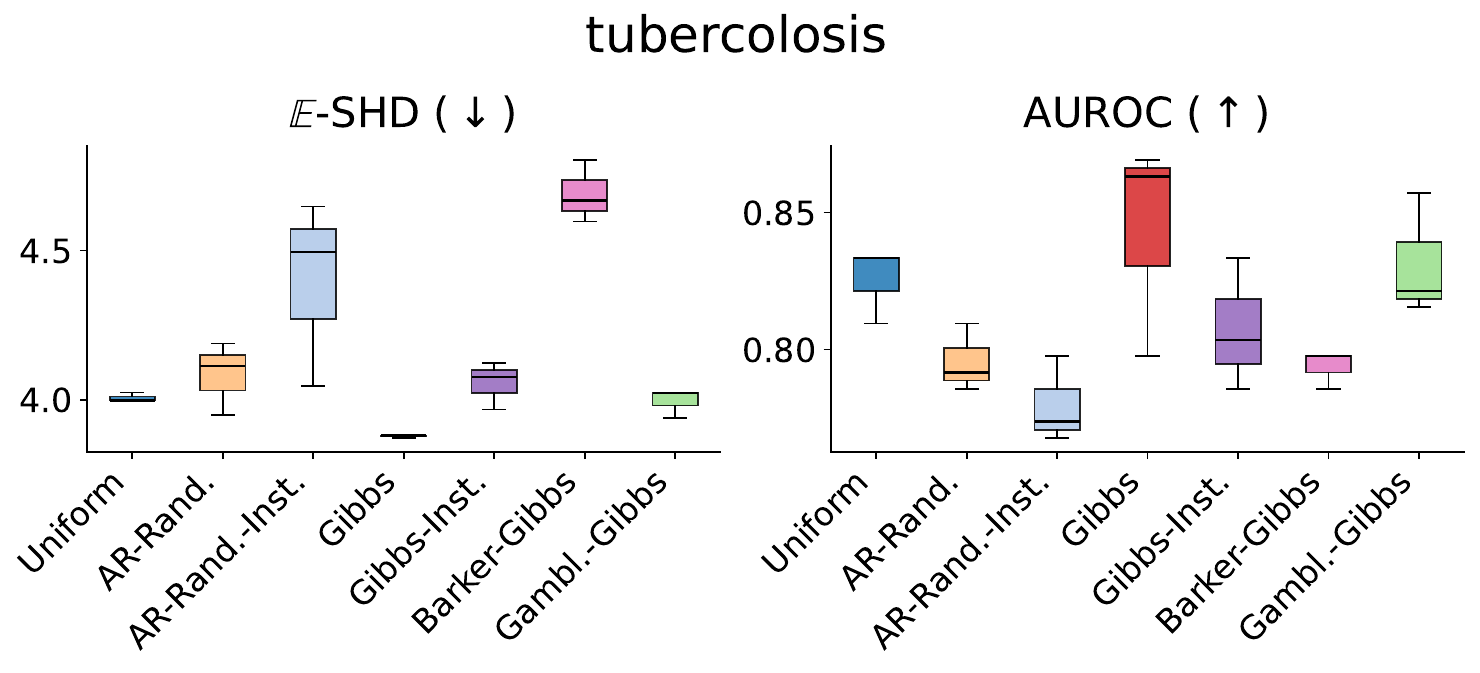}
  \end{subfigure}
  \begin{subfigure}[b]{0.495\linewidth}
    \includegraphics[width=1.0\linewidth]{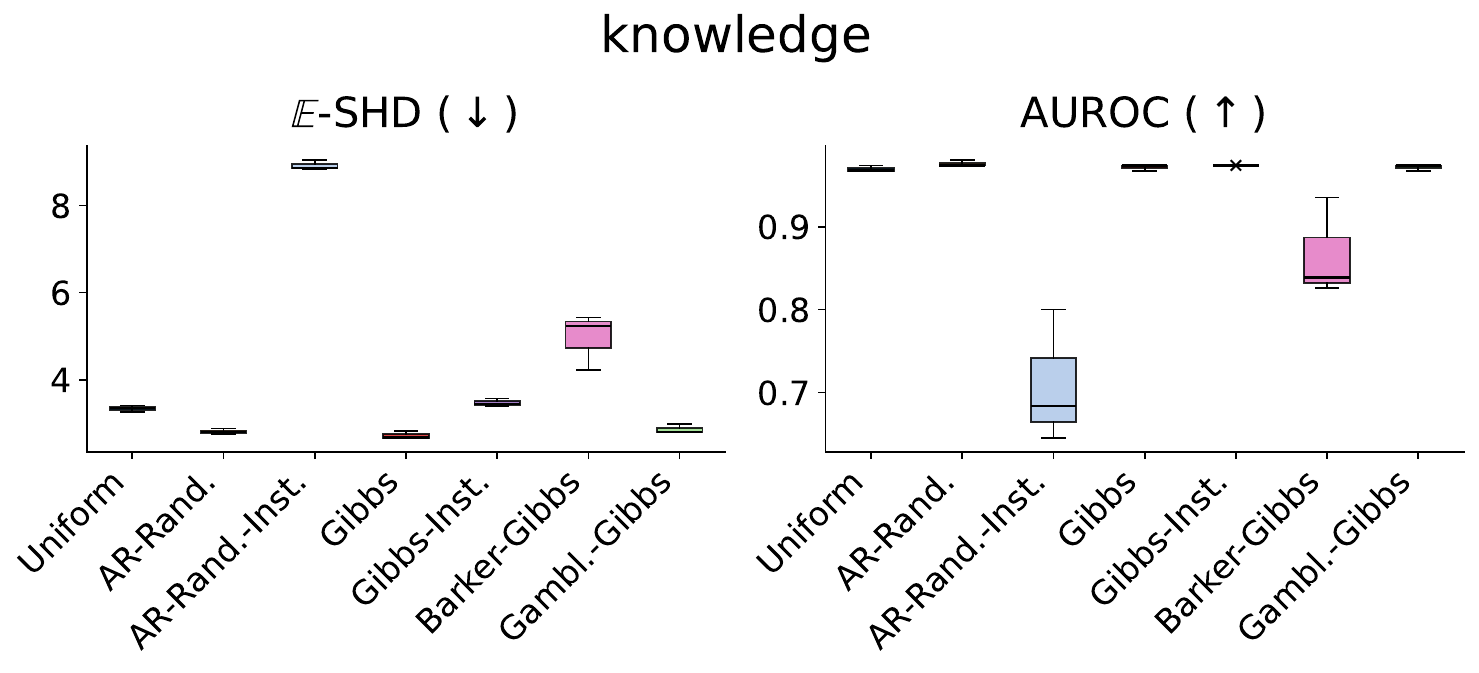}
  \end{subfigure}
  \begin{subfigure}[b]{0.495\linewidth}
    \includegraphics[width=1.0\linewidth]{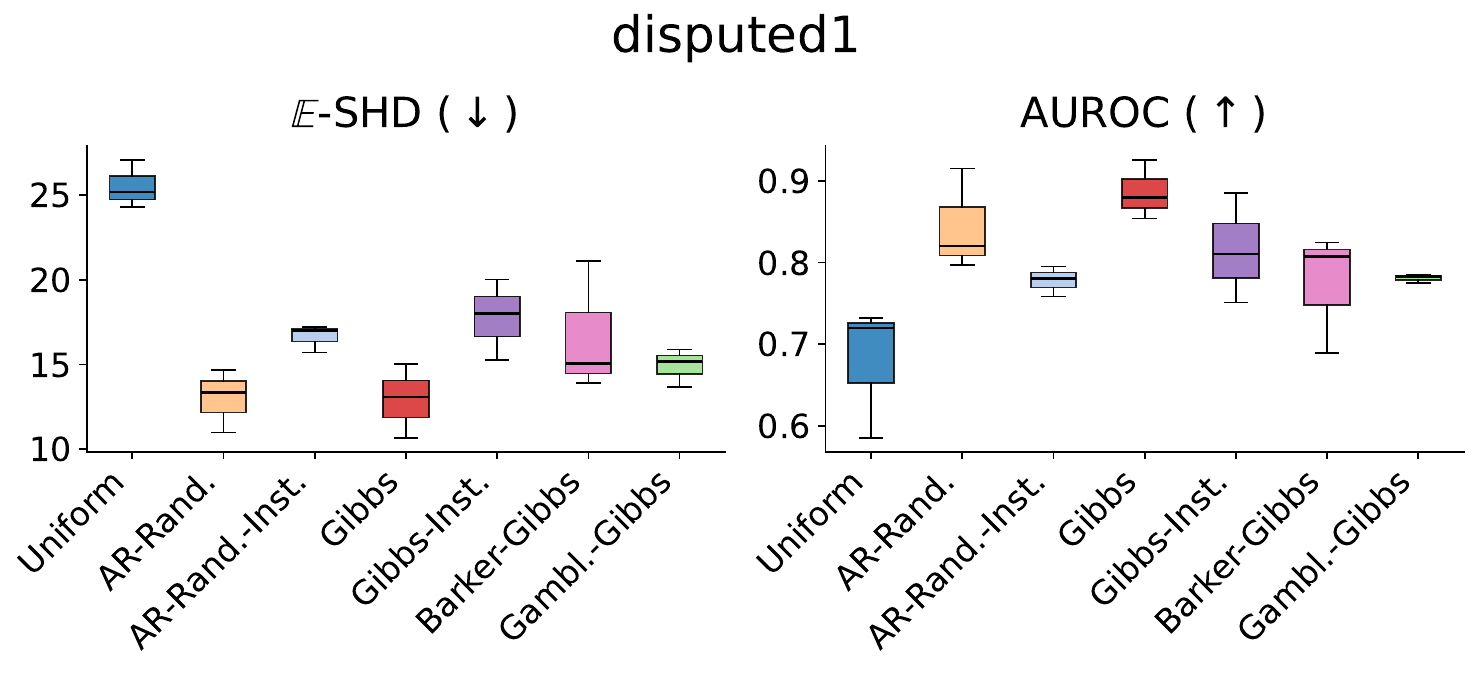}
  \end{subfigure}
  \begin{subfigure}[b]{0.495\linewidth}
    \includegraphics[width=1.0\linewidth]{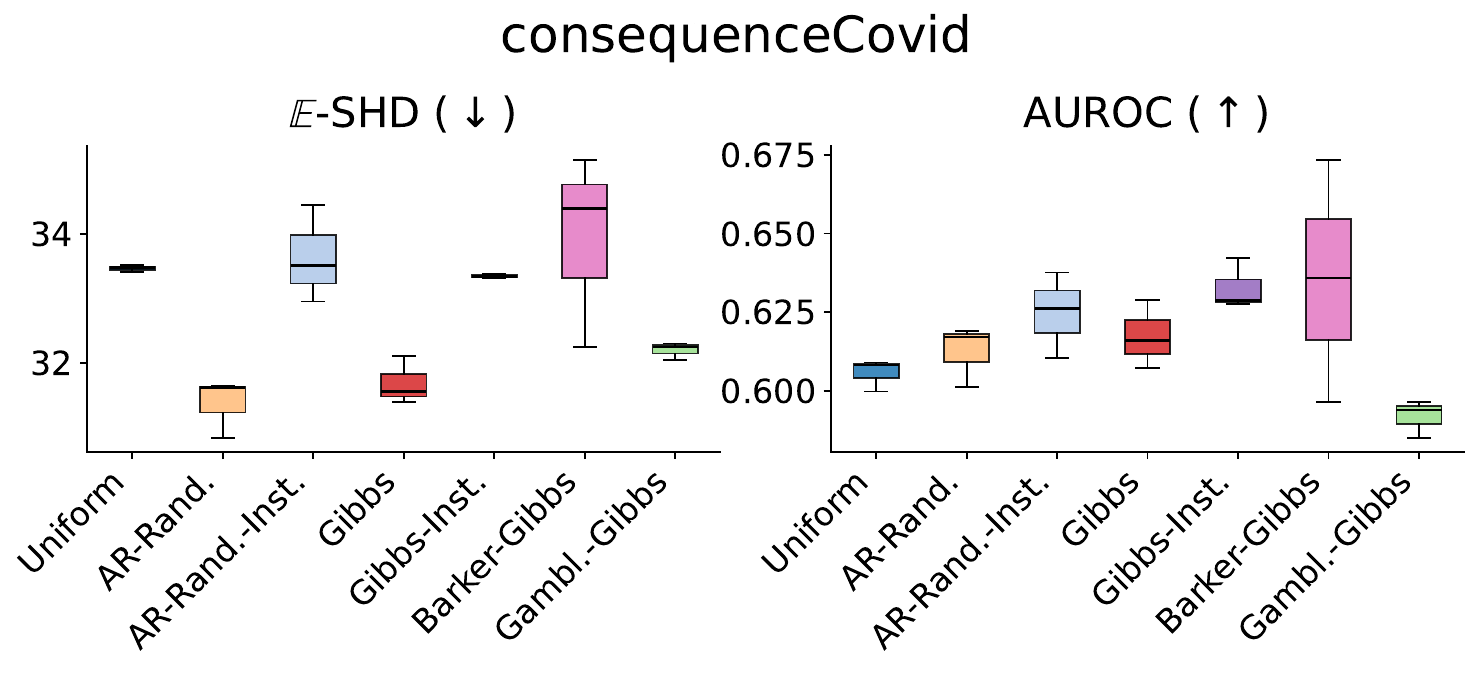}
  \end{subfigure}
  \caption{Bayesian structure learning results on four BnRep datasets using OLMo-3-32B (3 seeds).}
  \label{fig:structure_learning_olmo}
\end{figure*}

\subsection{Results from Different $\gamma$ Values}
\label{app:results:structure_learning:gamma}

In \eqref{eq:llm_prior}, the hyperparameter $\gamma$ controls the effective size of $D_{\text{LLM}}$. Under the Bayesian Dirichlet equivalent score \citep[BDe;][]{heckerman1995learning}, the effective size equals $\gamma \vert D_{\text{LLM}} \vert$. For the experiments in \cref{sec:structure_learning}, we fixed $\gamma= 0.5$ with $\vert D_{\text{LLM}} \vert =200$ and $\vert D \vert =100$, \ie, the effective size of the synthetic LLM dataset matches that of the ground truth dataset.

\cref{tab:structure_learning_gamma_ablation} presents additional results for $\gamma \in \{0.1, 0.2, 0.5, 1.0\}$ obtained via \textbf{AR-Random} and \textbf{Gibbs} sampling with Llama-3.1-8B (base). Overall, $\gamma = 0.2$ or $\gamma = 0.5$ yields the best results for each method, indicating that the synthetic LLM dataset is useful. When $\gamma = 1.0$, AUROC falls below the \textbf{Uniform} baseline in all settings except \texttt{disputed1}, indicating that overreliance on the synthetic data can make the downstream structure learning algorithm suboptimal, reaffirming that LLMs are \emph{imperfect} experts.

\begin{table}[h]
  \centering
  \resizebox{\linewidth}{!}{
    \begin{tabular}{@{}lcccccccc}
      \toprule
      Dataset $\rightarrow$ & \multicolumn{2}{c}{\texttt{tubercolosis}} & \multicolumn{2}{c}{\texttt{knowledge}} & \multicolumn{2}{c}{\texttt{disputed1}} & \multicolumn{2}{c}{\texttt{consequenceCovid}} \\
      \cmidrule(lr){2-3} \cmidrule(lr){4-5} \cmidrule(lr){6-7} \cmidrule(lr){8-9}
      Alg. $\downarrow$ Metric $\rightarrow$ 
      & $\mathbb{E}\text{-SHD}$ $(\downarrow)$ & AUROC $(\uparrow)$ 
      & $\mathbb{E}\text{-SHD}$ $(\downarrow)$ & AUROC $(\uparrow)$ 
      & $\mathbb{E}\text{-SHD}$ $(\downarrow)$ & AUROC $(\uparrow)$ 
      & $\mathbb{E}\text{-SHD}$ $(\downarrow)$ & AUROC $(\uparrow)$ \\
      \midrule
      Uniform (\ie, $\gamma=0$) 
      & 4.007\std{0.012} & 0.825\std{0.011}
      & 3.338\std{0.065} & 0.970\std{0.003}
      & 25.507\std{1.159} & 0.679\std{0.067}
      & 33.466\std{0.043} & 0.606\std{0.004}
      \\
      \midrule
      AR-Random, $\gamma=0.1$ 
      & 3.985\std{0.040} & 0.815\std{0.008}
      & 3.120\std{0.008} & 0.974\std{0.000}
      & 19.554\std{0.366} & 0.691\std{0.026}
      & 32.684\std{0.079} & 0.611\std{0.013}
      \\
      \phantom{AR-Random, }$\gamma=0.2$ 
      & 3.967\std{0.081} & 0.843\std{0.051}
      & 2.881\std{0.110} & 0.974\std{0.000}
      & 16.317\std{1.356} & 0.781\std{0.037}
      & 32.282\std{0.055} & 0.619\std{0.016}
      \\
      \phantom{AR-Random, }$\gamma=0.5$ 
      & 3.984\std{0.127} & 0.817\std{0.011}
      & 2.888\std{0.138} & 0.974\std{0.000}
      & 15.467\std{0.780} & 0.779\std{0.048}
      & 31.697\std{0.159} & 0.608\std{0.005}
      \\
      \phantom{AR-Random, }$\gamma=1.0$ 
      & 4.060\std{0.166} & 0.788\std{0.061}
      & 3.531\std{0.226} & 0.932\std{0.025}
      & 12.410\std{2.224} & 0.843\std{0.038}
      & 32.032\std{0.499} & 0.577\std{0.026}
      \\
      \midrule
      Gibbs, $\gamma=0.1$ 
      & 3.970\std{0.009} & 0.823\std{0.007}
      & 3.105\std{0.129} & 0.972\std{0.003}
      & 18.650\std{0.609} & 0.765\std{0.037}
      & 32.750\std{0.015} & 0.621\std{0.003}
      \\
      \phantom{Gibbs, }$\gamma=0.2$ 
      & 3.909\std{0.031} & 0.845\std{0.035}
      & 2.882\std{0.117} & 0.974\std{0.000}
      & 17.146\std{1.160} & 0.798\std{0.061}
      & 32.515\std{0.100} & 0.619\std{0.007}
      \\
      \phantom{Gibbs, }$\gamma=0.5$ 
      & 3.894\std{0.023} & 0.857\std{0.010}
      & 2.686\std{0.024} & 0.972\std{0.003}
      & 14.020\std{0.889} & 0.843\std{0.034}
      & 31.989\std{0.237} & 0.621\std{0.017}
      \\
      \phantom{Gibbs, }$\gamma=1.0$ 
      & 3.877\std{0.052} & 0.794\std{0.046}
      & 3.384\std{0.361} & 0.944\std{0.016}
      & 10.964\std{0.125} & 0.898\std{0.030}
      & 31.948\std{0.579} & 0.610\std{0.029}
      \\
      \bottomrule
    \end{tabular}
  }
  \caption{Results from different $\gamma$ values, using Llama-3.1-8B (base), mean$\pm$std over 3 seeds.}
  \label{tab:structure_learning_gamma_ablation}
\end{table}

\clearpage
\subsection{A Failure Case of Large Language Gibbs}
\label{app:failure_cases}

\begin{wraptable}{r}{0.3\textwidth}
\vspace*{-1.2em}
\centering
\begin{adjustbox}{width=0.7\linewidth, keepaspectratio}
\begin{tabular}{@{}lc}
\toprule
\textbf{Attribute} & \textbf{Possible values} \\
\midrule
C & low, high \\
Chl\_a & low, high \\
DO & low, high \\
N & low, high \\
P & low, high \\
pH & low, high \\
Te & low, high \\
Tu & low, high \\
\bottomrule
\end{tabular}
\end{adjustbox}
\caption{\label{tab:algalmetadata}Dataset metadata. The global description is: ``A dataset of discrete variables to predict chlorophyll-a (chl-a) using a range of water quality parameters as predictors''.}
\end{wraptable}

Here we show a case where augmenting Bayesian structure learning with large language Gibbs data fails to improve, and even hurts, structure discovery. We consider the \texttt{algalactivity2} dataset from BnRep, whose metadata is shown in \cref{tab:algalmetadata}. As we see, the metadata provides little information, and the possible values are similarly not linked to any ranges that could link them to a language model's world knowledge. In this setting, structure learning with a simple uniform prior recovers the graph better than when augmented with synthetic data generated with such a description (\cref{fig:failure_case}), which shows the importance of an informative data specification for the success of large language Gibbs.

\begin{figure*}[h]
    \vspace*{4em}
    \centering
    \includegraphics[width=0.6\linewidth]{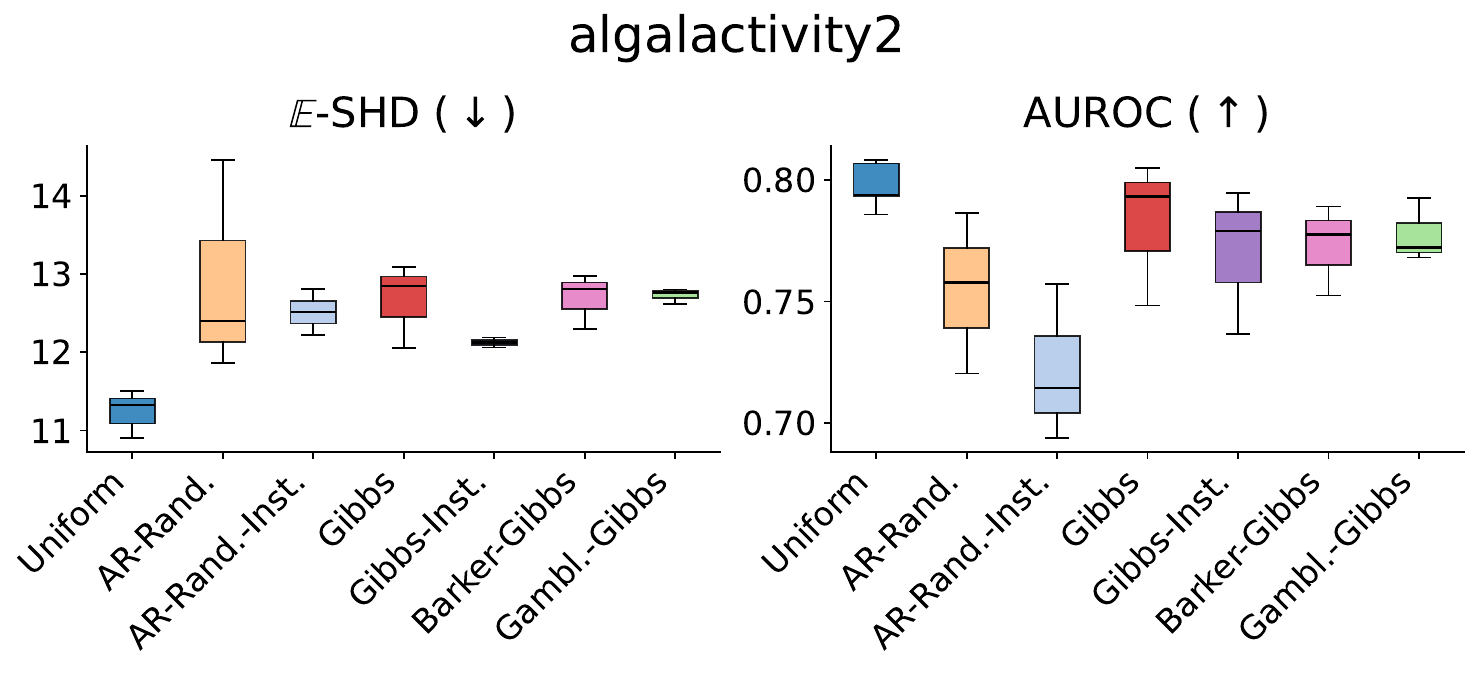}
    \caption{A failure case of Bayesian structure learning with synthetic data generated by large language Gibbs, cf.\ \cref{fig:structure_learning}.\label{fig:failure_case}}
\end{figure*}

\end{document}

%% file: figures/gibbs-cat.tex
\begin{tikzpicture}[remember picture,
    x=1cm,
    y=1cm,
    line width=1pt,
    every node/.style={inner sep=3pt, outer sep=3pt},
    box/.style={
        rectangle,
        draw=black,
        fill=white
    }
]\tt

\pgfdeclarelayer{background}
\pgfdeclarelayer{backerground}
\pgfsetlayers{backerground,background,main}

\node[box, fill=green!10] (metadata) at (0,0) {\metadatacontents};

\node[box, right=3.5cm of metadata, fill=purple!10] (sample1) {\samplecontents{shorthair}{tuxedo}{3}{12}};

\node[box, right=1cm of sample1, fill=purple!10] (sample2) {\samplecontents{shorthair}{\color{red}calico}{3}{12}};

\node[box, right=1cm of sample2, fill=purple!10] (sample3) {\samplecontents{shorthair}{calico}{3}{\color{red}14}};

\node[outer sep=3pt, right=1cm of sample3] (dots) {$\cdots$};

\node[box, right=1cm of dots, fill=purple!10] (sample4) {\samplecontents{\color{red}abyssinian}{tabby}{10}{15}};

\draw[->, >=stealth, decorate, decoration={snake, post length=4pt}] (sample1.east) -- (sample2.west) node[pos=0.5, above, anchor=south, inner sep=2pt] {\textrm{\large LM\vphantom{p}}};
\draw[->, >=stealth, decorate, decoration={snake, post length=4pt}] (sample2.east) -- (sample3.west) node[pos=0.5, above, anchor=south, inner sep=2pt] {\textrm{\large LM\vphantom{p}}};
\draw[->, >=stealth, decorate, decoration={snake, post length=4pt}] (sample3.east) -- (dots.west) node[pos=0.5, above, anchor=south, inner sep=2pt] {\textrm{\large LM\vphantom{p}}};
\draw[->, >=stealth, decorate, decoration={snake, post length=4pt}] (dots.east) -- (sample4.west) node[pos=0.5, above, anchor=south, inner sep=2pt] {\textrm{\large LM\vphantom{p}}};

\draw[->, >=stealth, decorate, decoration={snake, post length=4pt}] (metadata.east) -- (sample1.west) node[pos=0.5, above, anchor=south, inner sep=2pt] {\textrm{\large LM (autoreg.\ init.)}};

\node[above=0cm of metadata, inner sep=0pt,anchor=south] {\textrm{\large Metadata\vphantom{p}}};
\node[above=0cm of sample1, inner sep=0pt,anchor=south] {\textrm{\large Sample 1}};
\node[above=0cm of sample2, inner sep=0pt,anchor=south] {\textrm{\large Sample 2}};
\node[above=0cm of sample3, inner sep=0pt,anchor=south] {\textrm{\large Sample 3}};
\node[above=0cm of sample4, inner sep=0pt,anchor=south] {\textrm{\large Sample $n$}};

\coordinate (mid12) at ($(sample1.east)!0.5!(sample2.west) + (0,-0.2)$);
\coordinate (mid23) at ($(sample2.east)!0.5!(sample3.west) + (0,-0.2)$);
\coordinate (mid3d) at ($(sample3.east)!0.5!(dots.west) + (0,-0.2)$);
\coordinate (mid4d) at ($(dots.east)!0.5!(sample4.west) + (0,-0.2)$);
\coordinate (routebelow1) at ($(metadata.south) + (0,-0.5)$);
\coordinate (routebelow2) at ($(metadata.south) + (0,-0.7)$);
\coordinate (routebelow3) at ($(metadata.south) + (0,-0.9)$);
\coordinate (routebelow4) at ($(metadata.south) + (0,-1.1)$);

\draw[->, >=stealth] ([xshift=12pt,yshift=3pt]metadata.south) -- ([xshift=12pt]metadata.south |- routebelow1) -| (mid12) ;
\draw[->, >=stealth] ([xshift=4pt,yshift=3pt]metadata.south) -- ([xshift=4pt]metadata.south |- routebelow2) -|  (mid23) ;
\draw[->, >=stealth] ([xshift=-4pt,yshift=3pt]metadata.south) -- ([xshift=-4pt]metadata.south |- routebelow3) -| (mid3d) ;
\draw[->, >=stealth] ([xshift=-12pt,yshift=3pt]metadata.south) -- ([xshift=-12pt]metadata.south |- routebelow4) -| node[pos=0.25, left, anchor=north, inner sep=2pt] {\textrm{\large resampling prompt + constraints}} (mid4d) ;

\end{tikzpicture}

%% file: figures/gibbs-cr.tex
\begin{tikzpicture}[remember picture,
    x=1cm,
    y=1cm,
    line width=1pt,
    every node/.style={inner sep=3pt, outer sep=3pt},
    box/.style={
        rectangle,
        draw=black,
        fill=white
    }
]\tt

\pgfdeclarelayer{background}
\pgfdeclarelayer{backerground}
\pgfsetlayers{backerground,background,main}

\node[box, fill=green!10] (metadata) at (0,0) {\begin{tikzpicture}
        \node[anchor=west] (x1) at (0,-0) {The North Star is the brightest star in the night sky.\vphantom{pdl()}};
        \node[anchor=west] (x2) at (0,-0.5) {Sirius is the brightest star in the night sky.\vphantom{pdl()}};
        \node[anchor=west] (x3) at (0,-1) {Polaris is the brightest star in the night sky.\vphantom{pdl()}};
        \node[anchor=west] (x4) at (0,-1.5) {The sun is the brightest star in the night sky.\vphantom{pdl()}};
        \node[anchor=west] (m1) at (11,-0) {\color{blue}bool\vphantom{pdl()}};
        \node[anchor=west] (m2) at (11,-0.5) {\color{blue}bool\vphantom{pdl()}};
        \node[anchor=west] (m3) at (11,-1) {\color{blue}bool\vphantom{pdl()}};
        \node[anchor=west] (m4) at (11,-1.5) {\color{blue}bool\vphantom{pdl()}};
    \end{tikzpicture}};

\node[box, right=3.5cm of metadata, fill=purple!10] (sample1) {\samplecontents{True}{True}{False}{True}};

\node[box, right=1cm of sample1, fill=purple!10] (sample2) {\samplecontents{\color{red}False}{True}{False}{True}};

\node[box, right=1cm of sample2, fill=purple!10] (sample3) {\samplecontents{False}{True}{False}{\color{red}False}};

\node[outer sep=3pt, right=1cm of sample3] (dots) {$\cdots$};


\draw[->, >=stealth, decorate, decoration={snake, post length=4pt}] (sample1.east) -- (sample2.west) node[pos=0.5, above, anchor=south, inner sep=2pt] {\textrm{\large LM\vphantom{p}}};
\draw[->, >=stealth, decorate, decoration={snake, post length=4pt}] (sample2.east) -- (sample3.west) node[pos=0.5, above, anchor=south, inner sep=2pt] {\textrm{\large LM\vphantom{p}}};
\draw[->, >=stealth, decorate, decoration={snake, post length=4pt}] (sample3.east) -- (dots.west) node[pos=0.5, above, anchor=south, inner sep=2pt] {\textrm{\large LM\vphantom{p}}};

\draw[->, >=stealth, decorate, decoration={snake, post length=4pt}] (metadata.east) -- (sample1.west) node[pos=0.5, above, anchor=south, inner sep=2pt] {\textrm{\large LM (autoreg.\ init.)}};

\node[above=0cm of metadata, inner sep=0pt,anchor=south] {\textrm{\large \texttt{What is the brightest star in the night sky?}\vphantom{p}}};
\node[above=0cm of sample1, inner sep=0pt,anchor=south] {\textrm{\large Sample 1}};
\node[above=0cm of sample2, inner sep=0pt,anchor=south] {\textrm{\large Sample 2}};
\node[above=0cm of sample3, inner sep=0pt,anchor=south] {\textrm{\large Sample 3}};

\coordinate (mid12) at ($(sample1.east)!0.5!(sample2.west) + (0,-0.2)$);
\coordinate (mid23) at ($(sample2.east)!0.5!(sample3.west) + (0,-0.2)$);
\coordinate (mid3d) at ($(sample3.east)!0.5!(dots.west) + (0,-0.2)$);
\coordinate (routebelow1) at ($(metadata.south) + (0,-0.5)$);
\coordinate (routebelow2) at ($(metadata.south) + (0,-0.7)$);
\coordinate (routebelow3) at ($(metadata.south) + (0,-0.9)$);
\coordinate (routebelow4) at ($(metadata.south) + (0,-1.1)$);

\draw[->, >=stealth] ([xshift=12pt,yshift=3pt]metadata.south) -- ([xshift=12pt]metadata.south |- routebelow1) -| (mid12) ;
\draw[->, >=stealth] ([xshift=4pt,yshift=3pt]metadata.south) -- ([xshift=4pt]metadata.south |- routebelow2) -|  (mid23) ;
\draw[->, >=stealth] ([xshift=-4pt,yshift=3pt]metadata.south) -- ([xshift=-4pt]metadata.south |- routebelow3) -| node[pos=0.25, left, anchor=north, inner sep=2pt] {\textrm{\large resampling prompt + constraints}} (mid3d) ;

\end{tikzpicture}

%% file: references.bib
@inproceedings{devlin-etal-2019-bert,
    title = "{BERT}: Pre-training of Deep Bidirectional Transformers for Language Understanding",
    author = "Devlin, Jacob  and
      Chang, Ming-Wei  and
      Lee, Kenton  and
      Toutanova, Kristina",
    editor = "Burstein, Jill  and
      Doran, Christy  and
      Solorio, Thamar",
    booktitle = "Proceedings of the 2019 Conference of the North {A}merican Chapter of the Association for Computational Linguistics: Human Language Technologies, Volume 1 (Long and Short Papers)",
    month = jun,
    year = "2019",
    address = "Minneapolis, Minnesota",
    publisher = "Association for Computational Linguistics",
    url = "https://aclanthology.org/N19-1423/",
    doi = "10.18653/v1/N19-1423",
    pages = "4171--4186",
}

@article{domke2025large,
  title={Large language {Bayes}},
  author={Domke, Justin},
  journal={Neural Information Processing Systems (NeurIPS)},
  year={2025}
}

@inproceedings{kwon2023efficient,
  title={Efficient Memory Management for Large Language Model Serving with PagedAttention},
  author={Woosuk Kwon and Zhuohan Li and Siyuan Zhuang and Ying Sheng and Lianmin Zheng and Cody Hao Yu and Joseph E. Gonzalez and Hao Zhang and Ion Stoica},
  booktitle={Proceedings of the ACM SIGOPS 29th Symposium on Operating Systems Principles},
  year={2023}
}

@article{wei2022chain,
  title={Chain-of-thought prompting elicits reasoning in large language models},
  author={Wei, Jason and Wang, Xuezhi and Schuurmans, Dale and Bosma, Maarten and Ichter, Brian and Xia, Fei and Chi, Ed and Le, Quoc V and Zhou, Denny},
  journal={Neural Information Processing Systems (NeurIPS)},
  year={2022}
}

@article{grattafiori2024llama,
  title={The llama 3 herd of models},
  author={Aaron Grattafiori and Abhimanyu Dubey and Abhinav Jauhri and Abhinav Pandey and Abhishek Kadian and Ahmad Al-Dahle and Aiesha Letman and Akhil Mathur and Alan Schelten and Alex Vaughan and Amy Yang and Angela Fan and Anirudh Goyal and Anthony Hartshorn and Aobo Yang and Archi Mitra and Archie Sravankumar and Artem Korenev and Arthur Hinsvark and Arun Rao and Aston Zhang and Aurelien Rodriguez and Austen Gregerson and Ava Spataru and Baptiste Roziere and Bethany Biron and Binh Tang and Bobbie Chern and Charlotte Caucheteux and Chaya Nayak and Chloe Bi and Chris Marra and Chris McConnell and Christian Keller and Christophe Touret and Chunyang Wu and Corinne Wong and Cristian Canton Ferrer and Cyrus Nikolaidis and Damien Allonsius and Daniel Song and Danielle Pintz and Danny Livshits and Danny Wyatt and David Esiobu and Dhruv Choudhary and Dhruv Mahajan and Diego Garcia-Olano and Diego Perino and Dieuwke Hupkes and Egor Lakomkin and Ehab AlBadawy and Elina Lobanova and Emily Dinan and Eric Michael Smith and Filip Radenovic and Francisco Guzmán and Frank Zhang and Gabriel Synnaeve and Gabrielle Lee and Georgia Lewis Anderson and Govind Thattai and Graeme Nail and Gregoire Mialon and Guan Pang and Guillem Cucurell and Hailey Nguyen and Hannah Korevaar and Hu Xu and Hugo Touvron and Iliyan Zarov and Imanol Arrieta Ibarra and Isabel Kloumann and Ishan Misra and Ivan Evtimov and Jack Zhang and Jade Copet and Jaewon Lee and Jan Geffert and Jana Vranes and Jason Park and Jay Mahadeokar and Jeet Shah and Jelmer van der Linde and Jennifer Billock and Jenny Hong and Jenya Lee and Jeremy Fu and Jianfeng Chi and Jianyu Huang and Jiawen Liu and Jie Wang and Jiecao Yu and Joanna Bitton and Joe Spisak and Jongsoo Park and Joseph Rocca and Joshua Johnstun and Joshua Saxe and Junteng Jia and Kalyan Vasuden Alwala and Karthik Prasad and Kartikeya Upasani and Kate Plawiak and Ke Li and Kenneth Heafield and Kevin Stone and Khalid El-Arini and Krithika Iyer and Kshitiz Malik and Kuenley Chiu and Kunal Bhalla and Kushal Lakhotia and Lauren Rantala-Yeary and Laurens van der Maaten and Lawrence Chen and Liang Tan and Liz Jenkins and Louis Martin and Lovish Madaan and Lubo Malo and Lukas Blecher and Lukas Landzaat and Luke de Oliveira and Madeline Muzzi and Mahesh Pasupuleti and Mannat Singh and Manohar Paluri and Marcin Kardas and Maria Tsimpoukelli and Mathew Oldham and Mathieu Rita and Maya Pavlova and Melanie Kambadur and Mike Lewis and Min Si and Mitesh Kumar Singh and Mona Hassan and Naman Goyal and Narjes Torabi and Nikolay Bashlykov and Nikolay Bogoychev and Niladri Chatterji and Ning Zhang and Olivier Duchenne and Onur Çelebi and Patrick Alrassy and Pengchuan Zhang and Pengwei Li and Petar Vasic and Peter Weng and Prajjwal Bhargava and Pratik Dubal and Praveen Krishnan and Punit Singh Koura and Puxin Xu and Qing He and Qingxiao Dong and Ragavan Srinivasan and Raj Ganapathy and Ramon Calderer and Ricardo Silveira Cabral and Robert Stojnic and Roberta Raileanu and Rohan Maheswari and Rohit Girdhar and Rohit Patel and Romain Sauvestre and Ronnie Polidoro and Roshan Sumbaly and Ross Taylor and Ruan Silva and Rui Hou and Rui Wang and Saghar Hosseini and Sahana Chennabasappa and Sanjay Singh and Sean Bell and Seohyun Sonia Kim and Sergey Edunov and Shaoliang Nie and Sharan Narang and Sharath Raparthy and Sheng Shen and Shengye Wan and Shruti Bhosale and Shun Zhang and Simon Vandenhende and Soumya Batra and Spencer Whitman and Sten Sootla and Stephane Collot and Suchin Gururangan and Sydney Borodinsky and Tamar Herman and Tara Fowler and Tarek Sheasha and Thomas Georgiou and Thomas Scialom and Tobias Speckbacher and Todor Mihaylov and Tong Xiao and Ujjwal Karn and Vedanuj Goswami and Vibhor Gupta and Vignesh Ramanathan and Viktor Kerkez and Vincent Gonguet and Virginie Do and Vish Vogeti and Vítor Albiero and Vladan Petrovic and Weiwei Chu and Wenhan Xiong and Wenyin Fu and Whitney Meers and Xavier Martinet and Xiaodong Wang and Xiaofang Wang and Xiaoqing Ellen Tan and Xide Xia and Xinfeng Xie and Xuchao Jia and Xuewei Wang and Yaelle Goldschlag and Yashesh Gaur and Yasmine Babaei and Yi Wen and Yiwen Song and Yuchen Zhang and Yue Li and Yuning Mao and Zacharie Delpierre Coudert and Zheng Yan and Zhengxing Chen and Zoe Papakipos and Aaditya Singh and Aayushi Srivastava and Abha Jain and Adam Kelsey and Adam Shajnfeld and Adithya Gangidi and Adolfo Victoria and Ahuva Goldstand and Ajay Menon and Ajay Sharma and Alex Boesenberg and Alexei Baevski and Allie Feinstein and Amanda Kallet and Amit Sangani and Amos Teo and Anam Yunus and Andrei Lupu and Andres Alvarado and Andrew Caples and Andrew Gu and Andrew Ho and Andrew Poulton and Andrew Ryan and Ankit Ramchandani and Annie Dong and Annie Franco and Anuj Goyal and Aparajita Saraf and Arkabandhu Chowdhury and Ashley Gabriel and Ashwin Bharambe and Assaf Eisenman and Azadeh Yazdan and Beau James and Ben Maurer and Benjamin Leonhardi and Bernie Huang and Beth Loyd and Beto De Paola and Bhargavi Paranjape and Bing Liu and Bo Wu and Boyu Ni and Braden Hancock and Bram Wasti and Brandon Spence and Brani Stojkovic and Brian Gamido and Britt Montalvo and Carl Parker and Carly Burton and Catalina Mejia and Ce Liu and Changhan Wang and Changkyu Kim and Chao Zhou and Chester Hu and Ching-Hsiang Chu and Chris Cai and Chris Tindal and Christoph Feichtenhofer and Cynthia Gao and Damon Civin and Dana Beaty and Daniel Kreymer and Daniel Li and David Adkins and David Xu and Davide Testuggine and Delia David and Devi Parikh and Diana Liskovich and Didem Foss and Dingkang Wang and Duc Le and Dustin Holland and Edward Dowling and Eissa Jamil and Elaine Montgomery and Eleonora Presani and Emily Hahn and Emily Wood and Eric-Tuan Le and Erik Brinkman and Esteban Arcaute and Evan Dunbar and Evan Smothers and Fei Sun and Felix Kreuk and Feng Tian and Filippos Kokkinos and Firat Ozgenel and Francesco Caggioni and Frank Kanayet and Frank Seide and Gabriela Medina Florez and Gabriella Schwarz and Gada Badeer and Georgia Swee and Gil Halpern and Grant Herman and Grigory Sizov and Guangyi and Zhang and Guna Lakshminarayanan and Hakan Inan and Hamid Shojanazeri and Han Zou and Hannah Wang and Hanwen Zha and Haroun Habeeb and Harrison Rudolph and Helen Suk and Henry Aspegren and Hunter Goldman and Hongyuan Zhan and Ibrahim Damlaj and Igor Molybog and Igor Tufanov and Ilias Leontiadis and Irina-Elena Veliche and Itai Gat and Jake Weissman and James Geboski and James Kohli and Janice Lam and Japhet Asher and Jean-Baptiste Gaya and Jeff Marcus and Jeff Tang and Jennifer Chan and Jenny Zhen and Jeremy Reizenstein and Jeremy Teboul and Jessica Zhong and Jian Jin and Jingyi Yang and Joe Cummings and Jon Carvill and Jon Shepard and Jonathan McPhie and Jonathan Torres and Josh Ginsburg and Junjie Wang and Kai Wu and Kam Hou U and Karan Saxena and Kartikay Khandelwal and Katayoun Zand and Kathy Matosich and Kaushik Veeraraghavan and Kelly Michelena and Keqian Li and Kiran Jagadeesh and Kun Huang and Kunal Chawla and Kyle Huang and Lailin Chen and Lakshya Garg and Lavender A and Leandro Silva and Lee Bell and Lei Zhang and Liangpeng Guo and Licheng Yu and Liron Moshkovich and Luca Wehrstedt and Madian Khabsa and Manav Avalani and Manish Bhatt and Martynas Mankus and Matan Hasson and Matthew Lennie and Matthias Reso and Maxim Groshev and Maxim Naumov and Maya Lathi and Meghan Keneally and Miao Liu and Michael L. Seltzer and Michal Valko and Michelle Restrepo and Mihir Patel and Mik Vyatskov and Mikayel Samvelyan and Mike Clark and Mike Macey and Mike Wang and Miquel Jubert Hermoso and Mo Metanat and Mohammad Rastegari and Munish Bansal and Nandhini Santhanam and Natascha Parks and Natasha White and Navyata Bawa and Nayan Singhal and Nick Egebo and Nicolas Usunier and Nikhil Mehta and Nikolay Pavlovich Laptev and Ning Dong and Norman Cheng and Oleg Chernoguz and Olivia Hart and Omkar Salpekar and Ozlem Kalinli and Parkin Kent and Parth Parekh and Paul Saab and Pavan Balaji and Pedro Rittner and Philip Bontrager and Pierre Roux and Piotr Dollar and Polina Zvyagina and Prashant Ratanchandani and Pritish Yuvraj and Qian Liang and Rachad Alao and Rachel Rodriguez and Rafi Ayub and Raghotham Murthy and Raghu Nayani and Rahul Mitra and Rangaprabhu Parthasarathy and Raymond Li and Rebekkah Hogan and Robin Battey and Rocky Wang and Russ Howes and Ruty Rinott and Sachin Mehta and Sachin Siby and Sai Jayesh Bondu and Samyak Datta and Sara Chugh and Sara Hunt and Sargun Dhillon and Sasha Sidorov and Satadru Pan and Saurabh Mahajan and Saurabh Verma and Seiji Yamamoto and Sharadh Ramaswamy and Shaun Lindsay and Shaun Lindsay and Sheng Feng and Shenghao Lin and Shengxin Cindy Zha and Shishir Patil and Shiva Shankar and Shuqiang Zhang and Shuqiang Zhang and Sinong Wang and Sneha Agarwal and Soji Sajuyigbe and Soumith Chintala and Stephanie Max and Stephen Chen and Steve Kehoe and Steve Satterfield and Sudarshan Govindaprasad and Sumit Gupta and Summer Deng and Sungmin Cho and Sunny Virk and Suraj Subramanian and Sy Choudhury and Sydney Goldman and Tal Remez and Tamar Glaser and Tamara Best and Thilo Koehler and Thomas Robinson and Tianhe Li and Tianjun Zhang and Tim Matthews and Timothy Chou and Tzook Shaked and Varun Vontimitta and Victoria Ajayi and Victoria Montanez and Vijai Mohan and Vinay Satish Kumar and Vishal Mangla and Vlad Ionescu and Vlad Poenaru and Vlad Tiberiu Mihailescu and Vladimir Ivanov and Wei Li and Wenchen Wang and Wenwen Jiang and Wes Bouaziz and Will Constable and Xiaocheng Tang and Xiaojian Wu and Xiaolan Wang and Xilun Wu and Xinbo Gao and Yaniv Kleinman and Yanjun Chen and Ye Hu and Ye Jia and Ye Qi and Yenda Li and Yilin Zhang and Ying Zhang and Yossi Adi and Youngjin Nam and Yu and Wang and Yu Zhao and Yuchen Hao and Yundi Qian and Yunlu Li and Yuzi He and Zach Rait and Zachary DeVito and Zef Rosnbrick and Zhaoduo Wen and Zhenyu Yang and Zhiwei Zhao and Zhiyu Ma},
  journal={arXiv preprint arXiv:2407.21783},
  year={2024}
}

@article{olmo2025olmo,
  title={Olmo 3},
  author={Team Olmo and Allyson Ettinger and Amanda Bertsch and Bailey Kuehl and David Graham and David Heineman and Dirk Groeneveld and Faeze Brahman and Finbarr Timbers and Hamish Ivison and Jacob Morrison and Jake Poznanski and Kyle Lo and Luca Soldaini and Matt Jordan and Mayee Chen and Michael Noukhovitch and Nathan Lambert and Pete Walsh and Pradeep Dasigi and Robert Berry and Saumya Malik and Saurabh Shah and Scott Geng and Shane Arora and Shashank Gupta and Taira Anderson and Teng Xiao and Tyler Murray and Tyler Romero and Victoria Graf and Akari Asai and Akshita Bhagia and Alexander Wettig and Alisa Liu and Aman Rangapur and Chloe Anastasiades and Costa Huang and Dustin Schwenk and Harsh Trivedi and Ian Magnusson and Jaron Lochner and Jiacheng Liu and Lester James V. Miranda and Maarten Sap and Malia Morgan and Michael Schmitz and Michal Guerquin and Michael Wilson and Regan Huff and Ronan Le Bras and Rui Xin and Rulin Shao and Sam Skjonsberg and Shannon Zejiang Shen and Shuyue Stella Li and Tucker Wilde and Valentina Pyatkin and Will Merrill and Yapei Chang and Yuling Gu and Zhiyuan Zeng and Ashish Sabharwal and Luke Zettlemoyer and Pang Wei Koh and Ali Farhadi and Noah A. Smith and Hannaneh Hajishirzi},
  journal={arXiv preprint arXiv:2512.13961},
  year={2025}
}

@inproceedings{wang-cho-2019-bert,
    title = "{BERT} has a Mouth, and It Must Speak: {BERT} as a {M}arkov Random Field Language Model",
    author = "Wang, Alex  and
      Cho, Kyunghyun",
    editor = "Bosselut, Antoine  and
      Celikyilmaz, Asli  and
      Ghazvininejad, Marjan  and
      Iyer, Srinivasan  and
      Khandelwal, Urvashi  and
      Rashkin, Hannah  and
      Wolf, Thomas",
    booktitle = "Proceedings of the Workshop on Methods for Optimizing and Evaluating Neural Language Generation",
    month = jun,
    year = "2019",
    address = "Minneapolis, Minnesota",
    publisher = "Association for Computational Linguistics",
    url = "https://aclanthology.org/W19-2304/",
    doi = "10.18653/v1/W19-2304",
    pages = "30--36",
}

@inproceedings{yamakoshi-etal-2022-probing,
    title = "Probing {BERT}{'}s priors with serial reproduction chains",
    author = "Yamakoshi, Takateru  and
      Griffiths, Thomas  and
      Hawkins, Robert",
    editor = "Muresan, Smaranda  and
      Nakov, Preslav  and
      Villavicencio, Aline",
    booktitle = "Findings of the Association for Computational Linguistics: ACL 2022",
    month = may,
    year = "2022",
    address = "Dublin, Ireland",
    publisher = "Association for Computational Linguistics",
    url = "https://aclanthology.org/2022.findings-acl.314/",
    doi = "10.18653/v1/2022.findings-acl.314",
    pages = "3977--3992",
}

@article{
    goyal2022exposing,
    title={Exposing the Implicit Energy Networks behind Masked Language Models via Metropolis--Hastings},
    author={Kartik Goyal and Chris Dyer and Taylor Berg-Kirkpatrick},
    journal={International Conference on Learning Representations (ICLR)},
    year={2022},
}

@article{karan2025reasoning,
  title={Reasoning with sampling: Your base model is smarter than you think},
  author={Karan, Aayush and Du, Yilun},
  journal={International Conference on Learning Representations (ICLR)},
  year={2026}
}

@article{faria2024quest,
  title={QUEST: Quality-aware metropolis-hastings sampling for machine translation},
  author={Faria, Gon{\c{c}}alo R and Agrawal, Sweta and Farinhas, Ant{\'o}nio and Rei, Ricardo and de Souza, Jos{\'e} G and Martins, Andr{\'e} F},
  journal={Neural Information Processing Systems (NeurIPS)},
  year={2024}
}

@article{xiao2025flipping,
  title={Flipping against all odds: Reducing {LLM} coin flip bias via verbalized rejection sampling},
  author={Xiao, Tim Z and Zenn, Johannes and Liu, Zhen and Liu, Weiyang and Bamler, Robert and Sch{\"o}lkopf, Bernhard},
  journal={arXiv preprint arXiv:2506.09998},
  year={2025}
}

@article{arnold1998distributions,
  title={Distributions most nearly compatible with given families of conditional distributions},
  author={Arnold, Barry C and Gokhale, Dattaprabhakar V},
  journal={Test},
  volume={7},
  number={2},
  pages={377--390},
  year={1998},
  publisher={Springer}
}

@article{barker1965monte,
  title={{Monte Carlo} calculations of the radial distribution functions for a proton-electron plasma},
  author={Barker, Anthony Alfred},
  journal={Australian Journal of Physics},
  volume={18},
  number={2},
  pages={119--134},
  year={1965},
  publisher={CSIRO Publishing}
}

@article{gu2026illusion,
  title={The Illusion of Stochasticity in {LLMs}},
  author={Gu, Xiangming and De, Soham and Titsias, Michalis and Markeeva, Larisa and Veli{\v{c}}kovi{\'c}, Petar and Pascanu, Razvan},
  journal={arXiv preprint arXiv:2604.06543},
  year={2026}
}

@article{
hu2024amortizing,
title={Amortizing intractable inference in large language models},
author={Edward J Hu and Moksh Jain and Eric Elmoznino and Younesse Kaddar and Guillaume Lajoie and Yoshua Bengio and Nikolay Malkin},
journal={International Conference on Learning Representations (ICLR)},
year={2024},
}

@article{
hopkins2023can,
title={Can {LLM}s Generate Random Numbers? Evaluating {LLM} Sampling in Controlled Domains},
author={Aspen K Hopkins and Alex Renda and Michael Carbin},
journal={ICML 2023 Workshop: Sampling and Optimization in Discrete Space},
year={2023},
url={https://openreview.net/forum?id=Vhh1K9LjVI}
}

@article{kirkpatrick1983optimization,
  title={Optimization by simulated annealing},
  author={Kirkpatrick, Scott and Gelatt Jr, C Daniel and Vecchi, Mario P},
  journal={science},
  volume={220},
  number={4598},
  pages={671--680},
  year={1983},
  publisher={American association for the advancement of science}
}

@article{deleu2022bayesian,
  title={Bayesian structure learning with generative flow networks},
  author={Deleu, Tristan and G{\'o}is, Ant{\'o}nio and Emezue, Chris and Rankawat, Mansi and Lacoste-Julien, Simon and Bauer, Stefan and Bengio, Yoshua},
  journal={Uncertainty in Artificial Intelligence (UAI)},
  year={2022},
}

@article{heckerman1995learning,
  title={Learning {Bayesian} networks: The combination of knowledge and statistical data},
  author={Heckerman, David and Geiger, Dan and Chickering, David M},
  journal={Machine learning},
  volume={20},
  number={3},
  pages={197--243},
  year={1995},
  publisher={Springer}
}

@book{koller2009probabilistic,
  title={Probabilistic graphical models: principles and techniques},
  author={Koller, Daphne and Friedman, Nir},
  year={2009},
  publisher={MIT press}
}

@article{leonelli2025bnrep,
  title={{bnRep}: A repository of {Bayesian} networks from the academic literature},
  author={Leonelli, Manuele},
  journal={Neurocomputing},
  volume={624},
  pages={129502},
  year={2025},
  publisher={Elsevier}
}

@article{scutari2010learning,
  title={Learning {Bayesian} networks with the bnlearn {R} package},
  author={Scutari, Marco},
  journal={Journal of statistical software},
  volume={35},
  pages={1--22},
  year={2010}
}

@article{
gadetsky2025large,
title={Large (Vision) Language Models are Unsupervised In-Context Learners},
author={Artyom Gadetsky and Andrei Atanov and Yulun Jiang and Zhitong Gao and Ghazal Hosseini Mighan and Amir Zamir and Maria Brbic},
journal={International Conference on Learning Representations (ICLR)},
year={2025},
}

@article{wen2025unsupervised,
  title={Unsupervised elicitation of language models},
  author={Jiaxin Wen and Zachary Ankner and Yanda Chen and Arushi Somani and Peter Hase and Fabien Roger and Samuel Marks and Jacob Goldman-Wetzler and Linda Petrini and Henry Sleight and Collin Burns and He He and Shi Feng and Ethan Perez and Jan Leike},
  journal={arXiv preprint arXiv:2506.10139},
  year={2025}
}

@article{brown2020language,
  title={Language models are few-shot learners},
  author={Tom B. Brown and Benjamin Mann and Nick Ryder and Melanie Subbiah and Jared Kaplan and Prafulla Dhariwal and Arvind Neelakantan and Pranav Shyam and Girish Sastry and Amanda Askell and Sandhini Agarwal and Ariel Herbert-Voss and Gretchen Krueger and Tom Henighan and Rewon Child and Aditya Ramesh and Daniel M. Ziegler and Jeffrey Wu and Clemens Winter and Christopher Hesse and Mark Chen and Eric Sigler and Mateusz Litwin and Scott Gray and Benjamin Chess and Jack Clark and Christopher Berner and Sam McCandlish and Alec Radford and Ilya Sutskever and Dario Amodei},
  journal={Neural Information Processing Systems (NeurIPS)},
  year={2020}
}

@inproceedings{torroba-hennigen-kim-2023-deriving,
    title = "Deriving Language Models from Masked Language Models",
    author = "Torroba Hennigen, Lucas  and
      Kim, Yoon",
    editor = "Rogers, Anna  and
      Boyd-Graber, Jordan  and
      Okazaki, Naoaki",
    booktitle = "Proceedings of the 61st Annual Meeting of the Association for Computational Linguistics (Volume 2: Short Papers)",
    month = jul,
    year = "2023",
    address = "Toronto, Canada",
    publisher = "Association for Computational Linguistics",
    url = "https://aclanthology.org/2023.acl-short.99/",
    doi = "10.18653/v1/2023.acl-short.99",
    pages = "1149--1159"
}

@inproceedings{lu-etal-2022-fantastically,
    title = "Fantastically Ordered Prompts and Where to Find Them: Overcoming Few-Shot Prompt Order Sensitivity",
    author = "Lu, Yao  and
      Bartolo, Max  and
      Moore, Alastair  and
      Riedel, Sebastian  and
      Stenetorp, Pontus",
    editor = "Muresan, Smaranda  and
      Nakov, Preslav  and
      Villavicencio, Aline",
    booktitle = "Proceedings of the 60th Annual Meeting of the Association for Computational Linguistics (Volume 1: Long Papers)",
    month = may,
    year = "2022",
    address = "Dublin, Ireland",
    publisher = "Association for Computational Linguistics",
    url = "https://aclanthology.org/2022.acl-long.556/",
    doi = "10.18653/v1/2022.acl-long.556",
    pages = "8086--8098",
}

@inproceedings{zhu2024recovering,
  title={Recovering Mental Representations from Large Language Models with Markov Chain Monte Carlo},
  author={Zhu, Jian-Qiao and Yan, Haijiang and Griffiths, Tom},
  booktitle={Proceedings of the Annual Meeting of the Cognitive Science Society},
  volume={46},
  year={2024}
}

@inproceedings{zhu2025eliciting,
  title={Eliciting the Priors of Large Language Models using Iterated In-Context Learning},
  author={Zhu, Jian-Qiao and Griffiths, Tom},
  booktitle={Proceedings of the Annual Meeting of the Cognitive Science Society},
  volume={47},
  year={2025}
}

@article{borisov2023language,
    title={Language Models are Realistic Tabular Data Generators}, 
    author={Vadim Borisov and Kathrin Seßler and Tobias Leemann and Martin Pawelczyk and Gjergji Kasneci},
    year={2023},
    journal={International Conference on Learning Representations (ICLR)}
}

@article{liu-etal-2024-lost,
    title = "Lost in the Middle: How Language Models Use Long Contexts",
    author = "Liu, Nelson F.  and
      Lin, Kevin  and
      Hewitt, John  and
      Paranjape, Ashwin  and
      Bevilacqua, Michele  and
      Petroni, Fabio  and
      Liang, Percy",
    journal = "Transactions of the Association for Computational Linguistics",
    volume = "12",
    year = "2024",
    address = "Cambridge, MA",
    publisher = "MIT Press",
    url = "https://aclanthology.org/2024.tacl-1.9/",
    doi = "10.1162/tacl_a_00638",
    pages = "157--173",
}

@article{geman1984stochastic,
  title={Stochastic relaxation, {Gibbs} distributions, and the {Bayesian} restoration of images},
  author={Geman, Stuart and Geman, Donald},
  journal={IEEE Transactions on pattern analysis and machine intelligence},
  number={6},
  pages={721--741},
  year={1984},
  publisher={IEEE}
}

@book{mackay2003information,
title={Information theory, inference and learning algorithms},
author={MacKay, David JC},
year={2003},
publisher={Cambridge university press}
}

@article{
swamy2026all,
title={All Roads Lead to Likelihood: The Value of Reinforcement Learning in Fine-Tuning},
author={Gokul Swamy and Sanjiban Choudhury and Wen Sun and Steven Wu and Drew Bagnell},
journal={International Conference on Learning Representations (ICLR)},
year={2026},
}

@article{xu2024reprompting,
    title={Reprompting: Automated Chain-of-Thought Prompt Inference Through {Gibbs} Sampling}, 
    author={Weijia Xu and Andrzej Banburski-Fahey and Nebojsa Jojic},
    year={2024},
    journal={International Conference on Machine Learning (ICML)}
}

@article{chickering2003optimal,
author = {Chickering, David Maxwell},
title = {Optimal structure identification with greedy search},
year = {2003},
volume = {3},
journal = {Journal of Machine Learning Research},
month = 3,
pages = {507--554},
}

@book{spirtes,
author = {Spirtes, Peter and Glymour, Clark and Scheines, Richard},
year = {1993},
title = {Causation, Prediction, and Search}
}

@article{lorch2022amortized,
  title={{DiBS}: Differentiable {Bayesian} Structure Learning},
  author={Lorch, Lars and Rothfuss, Jonas and Krause, Andreas and Scholkopf, Bernhard},
  journal={Neural Information Processing Systems (NeurIPS)},
  year={2022}
}

@article{viinikka2021towards,
  title={Towards scalable {Bayesian} learning of causal {DAGs}},
  author={Viinikka, Jussi and Hyttinen, Antti and Pensar, Johan and Koivisto, Mikko},
  journal={Neural Information Processing Systems (NeurIPS)},
  year={2021}
}

@article{madigan1995bayesian,
  title={Bayesian graphical models for discrete data},
  author={Madigan, David and York, Jeremy and Allard, Denis},
  journal={International Statistical Review},
  year={1995},
    volume={63},
    number={2},
    pages={215--232},
}

@article{giudici2003improving,
  title={Improving {Markov chain Monte Carlo} model search for data mining},
  author={Giudici, Paolo and Castelo, Robert},
  journal={Machine Learning},
  volume={50},
  pages={127--158},
  year={2003},
}

@article{heckerman2000dependency,
  title={Dependency networks for inference, collaborative filtering, and data visualization},
  author={Heckerman, David and Chickering, David Maxwell and Meek, Christopher and Rounthwaite, Robert and Kadie, Carl},
  journal={Journal of Machine Learning Research},
  volume={1},
  number={Oct},
  pages={49--75},
  year={2000}
}

@article{selby2025had,
  title={Had enough of experts? quantitative knowledge retrieval from large language models},
  author={Selby, David and Iwashita, Yuichiro and Spriestersbach, Kai and Saad, Mohammad and Bappert, Dennis and Warrier, Archana and Mukherjee, Sumantrak and Kise, Koichi and Vollmer, Sebastian},
  journal={Stat},
  volume={14},
  number={2},
  pages={e70054},
  year={2025},
  publisher={Wiley Online Library}
}

@article{darvariu2024large,
  title={Large language models are effective priors for causal graph discovery},
  author={Darvariu, Victor-Alexandru and Hailes, Stephen and Musolesi, Mirco},
  journal={arXiv preprint arXiv:2405.13551},
  year={2024}
}

@article{capstick2024autoelicit,
  title={{AutoElicit}: Using large language models for expert prior elicitation in predictive modelling},
  author={Capstick, Alexander and Krishnan, Rahul G and Barnaghi, Payam},
  journal={International Conference on Machine Learning (ICML)},
  year={2025}
}

@article{arai2025many,
  title={How many patients could we save with {LLM} priors?},
  author={Arai, Shota and Selby, David and Vargo, Andrew and Vollmer, Sebastian},
  journal={arXiv preprint arXiv:2509.04250},
  year={2025}
}

@inproceedings{
gouk2024automated,
title={Automated Prior Elicitation from Large Language Models for {Bayesian} Logistic Regression},
author={Henry Gouk and Boyan Gao},
booktitle={AutoML Conference 2024 (Workshop Track)},
year={2024},
url={https://openreview.net/forum?id=euLzlnU7gz}
}

@article{du2024principled,
  title={Principled Gradient-based Markov Chain Monte Carlo for Text Generation},
  author={Li Du and Afra Amini and Lucas Torroba Hennigen and Xinyan Velocity Yu and Jason Eisner and Holden Lee and Ryan Cotterell},
  journal={International Conference on Machine Learning (ICML)},
  year={2024}
}

@article{miao2019cgmh,
  title={{CGMH}: Constrained Sentence Generation by {Metropolis-Hastings} Sampling},
    author={Ning Miao and Hao Zhou and Lili Mou and Rui Yan and Lei Li},
    journal={AAAI Conference on Artificial Intelligence},
    year={2019},
}

@article{qin2022cold,
  title={{COLD} decoding: Energy-based constrained text generation with {Langevin} dynamics},
  author={Lianhui Qin and Sean Welleck and Daniel Khashabi and Yejin Choi},
  journal={Neural Information Processing Systems (NeurIPS)},
  year={2022}
}

@inproceedings{kumar-etal-2022-gradient,
    title = "Gradient-based Constrained Sampling from Language Models",
    author = "Kumar, Sachin  and
      Paria, Biswajit  and
      Tsvetkov, Yulia",
    editor = "Goldberg, Yoav  and
      Kozareva, Zornitsa  and
      Zhang, Yue",
    booktitle = "Proceedings of the 2022 Conference on Empirical Methods in Natural Language Processing",
    month = dec,
    year = "2022",
    address = "Abu Dhabi, United Arab Emirates",
    publisher = "Association for Computational Linguistics",
    url = "https://aclanthology.org/2022.emnlp-main.144/",
    doi = "10.18653/v1/2022.emnlp-main.144",
    pages = "2251--2277",
}

@inproceedings{mireshghallah-etal-2022-mix,
    title = "Mix and Match: Learning-free Controllable Text Generationusing Energy Language Models",
    author = "Mireshghallah, Fatemehsadat  and
      Goyal, Kartik  and
      Berg-Kirkpatrick, Taylor",
    editor = "Muresan, Smaranda  and
      Nakov, Preslav  and
      Villavicencio, Aline",
    booktitle = "Proceedings of the 60th Annual Meeting of the Association for Computational Linguistics (Volume 1: Long Papers)",
    month = may,
    year = "2022",
    address = "Dublin, Ireland",
    publisher = "Association for Computational Linguistics",
    url = "https://aclanthology.org/2022.acl-long.31/",
    doi = "10.18653/v1/2022.acl-long.31",
    pages = "401--415",
}
